\documentclass[11pt]{article}
\usepackage[todo,authblk,nofancy]{anishs}
\usepackage{fullpage}
\usepackage[utf8]{inputenc}
\usepackage{amsmath}
\usepackage{amssymb}
\usepackage{cancel}

\newcommand{\E}{\mathbb{E}}
\newcommand{\id}{\mathrm{I}}

\newcommand{\Lip}{\mathrm{Lip}}

\DeclareMathOperator{\KL}{KL}

\def\shownotes{1}  
\ifnum\shownotes=1
\newcommand{\authnote}[2]{{$\ll$\textsf{\footnotesize #1 notes: #2}$\gg$}}
\else
\newcommand{\authnote}[2]{}
\fi
\newcommand{\Snote}[1]{{\color{magenta}\authnote{Anish:}{{#1}}}}

\newcommand{\timethree}{\phi}

\newcommand{\al}[0]{\alpha}
\newcommand{\be}[0]{\beta}

\newcommand{\ep}[0]{\epsilon}
\newcommand{\ga}[0]{\gamma}
\newcommand{\ka}[0]{\kappa}
\newcommand{\Om}[0]{\Omega}

\newcommand{\rh}[0]{\rho}
\newcommand{\si}[0]{\sigma}
\newcommand{\Si}[0]{\Sigma}
\newcommand{\ze}[0]{\zeta}

\newcommand{\td}[1]{\widetilde{#1}}

\newcommand{\ab}[1]{\left| {#1}\right|}
\newcommand{\an}[1]{\left\langle {#1}\right\rangle}
\newcommand{\ba}[1]{\left[ {#1} \right]}
\newcommand{\pa}[1]{\left( {#1} \right)}
\newcommand{\ve}[1]{\left\Vert {#1}\right\Vert}
\newcommand{\pf}[2]{\pa{\frac{#1}{#2}}}

\newcommand{\nb}[0]{\nabla}

\newcommand{\sub}[0]{\subset}

\newcommand{\ub}[2]{\underbrace{#1}_{#2}}

\newcommand{\ot}[0]{\otimes}
\newcommand{\iy}[0]{\infty}

\newcommand{\fc}[2]{\frac{#1}{#2}}
\newcommand{\sfc}[2]{\sqrt{\frac{#1}{#2}}}

\newcommand{\rc}[1]{\frac{1}{#1}}

\newcommand{\pl}[0]{\partial}

\newcommand{\ddd}[1]{\frac{d}{d #1}}

\newcommand{\coltwo}[2]{
\begin{bmatrix}
{#1}\\
{#2}
\end{bmatrix}}
\newcommand{\matt}[4]{
\begin{bmatrix}
{#1}&{#2}\\
{#3}&{#4}
\end{bmatrix}
}

\newcommand{\vocab}[1]{\textbf{#1}}

\title{Universal Approximation for Log-concave Distributions using Well-conditioned Normalizing Flows}
\author[1]{Holden Lee}
\author[2]{Chirag Pabbaraju}
\author[3]{Anish Sevekari}
\author[4]{Andrej Risteski}
\affil[1]{Duke University}
\affil[2,3,4]{Carnegie Mellon University}

\begin{document}
\maketitle

\begin{abstract}
Normalizing flows are a widely used class of latent-variable generative models with a tractable likelihood. Affine-coupling  models \citep{dinh2014nice, dinh2016density} are a particularly common type of normalizing flows, for which the Jacobian of the latent-to-observable-variable transformation is triangular, allowing the likelihood to be computed in linear time. 
Despite the widespread usage of affine couplings, the special structure of the architecture makes understanding their representational power challenging. The question of universal approximation was only recently resolved by three parallel papers \citep{huang2020augmented, zhang2020approximation, koehler2020representational} -- who showed reasonably regular distributions can be approximated arbitrarily well using affine couplings---albeit with networks with a nearly-singular Jacobian. As ill-conditioned Jacobians are an obstacle for likelihood-based training, the fundamental question remains: which distributions can be approximated using \emph{well-conditioned} affine coupling flows?     

In this paper, we show that any \emph{log-concave} distribution can be approximated using well-conditioned affine-coupling flows. In terms of proof techniques, we uncover and leverage deep connections between affine coupling architectures, underdamped Langevin dynamics (a stochastic differential equation often used to sample from Gibbs measures) and H\'enon maps (a structured dynamical system that appears in the study of symplectic diffeomorphisms).
Our results also inform the practice of training affine couplings: we approximate a padded version of the input distribution with iid Gaussians---a strategy which \cite{koehler2020representational} empirically observed to result in better-conditioned flows, but had hitherto no theoretical grounding. Our proof can thus be seen as providing theoretical evidence for the benefits of Gaussian padding when training normalizing flows. 
\end{abstract}


\section{Introduction}

Normalizing flows \citep{dinh2014nice,rezende2015variational} are a class of generative models parametrizing a distribution in $\mathbb{R}^d$ as the pushfoward of a simple distribution (e.g. Gaussian) through an invertible map $g_{\theta}: \mathbb{R}^d \to \mathbb{R}^d$ with trainable parameter $\theta$. The fact that $g_{\theta}$ is invertible allows us to write down an explicit expression for the density of a point $x$ through the change-of-variables formula, namely $p_{\theta}(x) = \phi(g^{-1}_{\theta}(x)) \mbox{det} (D g^{-1}_{\theta} (x))$, where $\phi$ denotes the density of the standard Gaussian. For different choices of parametric families for $g_{\theta}$, one gets different families of normalizing flows, e.g. affine coupling flows \citep{dinh2014nice,dinh2016density,kingma2018glow}, Gaussianization flows \citep{meng2020gaussianization}, sum-of-squares polynomial flows \citep{jaini2019sumofsquares}. 

In this paper we focus on affine coupling flows -- arguably the family that has been most successfully scaled up to high resolution datasets \citep{kingma2018glow}. The parametrization of $g_{\theta}$ is chosen to be a composition of so-called \emph{affine coupling blocks}, which are maps $f: \mathbb{R}^d \to \mathbb{R}^d$, s.t.   $f(x_S, x_{[d] \setminus S}) = (x_S, x_{[d] \setminus S} \odot s(x_S) + t(x_S))$, where $\odot$ denotes entrywise multiplication and $s, t$ are (typically simple) neural networks. The choice of parametrization is motivated by the fact that the Jacobian of each affine block is triangular, so that the determinant can be calculated in linear time.

Despite the empirical success of this architecture, theoretical understanding remains elusive. The most basic questions revolve around the representational power of such models. Even the question of universal approximation was only recently answered by three concurrent papers \citep{huang2020augmented, zhang2020approximation, koehler2020representational}---though in a less-than-satisfactory manner, in light of how normalizing flows are trained. Namely,  \cite{huang2020augmented, zhang2020approximation} show that any (reasonably well-behaved) distribution
$p$, once padded with zeros and treated as a distribution in $\R^{d+d'}$, can be arbitrarily closely approximated by an affine coupling flow.
While such padding can  be operationalized as an algorithm by padding the training image with zeros, it is never done in practice, as it results in an ill-conditioned Jacobian. This is expected, as the map that always sends the last $d'$ coordinates to 0 is not injective. \cite{koehler2020representational} prove universal approximation without padding; however their construction \emph{also} gives rise to a poorly conditioned Jacobian: namely, to approximate a distribution $p$ to within accuracy $\epsilon$ in the Wasserstein-1 distance, the Jacobian of the network they construct will have smallest singular value on the order of $\epsilon$. 

Importantly, for all these constructions, the condition number of the resulting affine coupling map is poor \emph{no matter how nice the underlying distribution it's trying to approximate is}. In other words, the source of this phenomenon isn't that the underlying distribution is low-dimensional or otherwise degenerate. 
Thus the question arises:
\vspace{-0.1in}\paragraph{Question:} \emph{Can well-behaved distributions be approximated by an affine coupling flow with a well-conditioned Jacobian?} 

In this paper, we answer the above question in the affirmative for a broad class of distributions -- log-concave distributions -- if we pad the input distribution not with zeroes, but with independent Gaussians. This gives theoretical grounding of an empirical observation in \cite{koehler2020representational} that Gaussian padding works better than zero-padding, as well as no padding. 

The practical relevance of this question is in providing guidance on the type of distributions we can hope to fit via training using an affine coupling flow. Theoretically, our techniques uncover some deep connections between affine coupling flows and two other (seeming unrelated) areas of mathematics: \emph{stochastic differential equations} (more precisely \emph{underdamped Langevin dynamics}, a ``momentum'' variant of the standard overdamped Langevin dynamics) and \emph{dynamical systems} (more precisely, a family of dynamical systems called \emph{H\'enon-like maps}).


\section{Overview of results}

In order to state our main result, we introduce some notation and definitions. 

\subsection{Notation} 
\label{s:notation}
\begin{definition} An \emph{affine coupling block} is a map $f: \mathbb{R}^d \to \mathbb{R}^d$, s.t.   $f(x_S, x_{[d] \setminus S}) = (x_S, x_{[d] \setminus S} \odot s(x_S) + t(x_S))$ for some set of coordinates $S$, where $\odot$ denotes entrywise multiplication and $s,t$ are trainable (generally non-linear) functions. 
An \emph{affine coupling network} is a finite sequence of affine coupling blocks. Note that the partition $(S, [d] \setminus S)$, as well as $s, t$ may be different between blocks. 
We say that the non-linearities are in a class $\mathcal{F}$ (e.g., neural networks, polynomials, etc.) if $s,t \in \mathcal{F}$. 
\label{eq:affineblock}
\end{definition}

The appeal of affine coupling networks comes from the fact that the Jacobian of each affine block is triangular, so calculating the determinant is a linear-time operation.  

We will be interested in the \emph{conditioning} of $f$---that is, an upper bound on the largest singular value $\sigma_{\max} (D f)$ and lower bound on the smallest singular value $\sigma_{\min}(D f)$ of the Jacobian $D f$ of $f$. Note that this is a slight abuse of nomenclature -- most of the time, ``condition number'' refers to the ratio of the largest and smallest singular value. As training a normalizing flow involves evaluating $\mbox{det}(D f)$, we in fact want to ensure that neither the smallest nor largest singular values are extreme. 

The class of distributions we will focus on approximating via affine coupling flows is \emph{log-concave} distributions:
\begin{definition}
A distribution $p: \mathbb{R}^d \to \mathbb{R}^{+},  p(x)\propto e^{-U(x)}$ is \emph{log-concave} if $\nb^2 U(x)=-\nabla^2 \ln p(x) \succeq 0$. 
\end{definition} 
Log-concave distributions are typically used to model distributions with Gaussian-like tail behavior. What we will leverage about this class of distributions is that a special stochastic differential equation (SDE), called \emph{underdamped Langevin dynamics}, is well-behaved in an analytic sense. 
Finally, we recall the definitions of positive definite matrices and Wasserstein distance, and introduce a notation for truncated distributions.

\begin{definition}
We say that a symmetric matrix is \emph{positive semidefinite (PSD)} if all of its eigenvalues are non-negative. For symmetric matrices $A,B$, we write $A \succeq B$ if and only if $A - B$ is PSD. 
\end{definition}

\begin{definition}
Given two probability measures $\mu,\nu$ over a metric space $(M,d)$, the \emph{Wasserstein-$1$} distance between them, denoted $W_1(\mu, \nu)$, is defined as
\[ W_1(\mu,\nu) = \inf_{\gamma \in \Gamma(\mu,\nu)} \int_{M \times M} d(x,y)\, d\gamma(x,y) \]
where $\Gamma(\mu,\nu)$ is the set of couplings, i.e. measures on $M \times M$ with marginals $\mu,\nu$ respectively. For two probability \emph{distributions} $p,q$, we denote by $W_1(p,q)$ the Wasserstein-$1$ distance between their associated measures. In this paper, we set $M=\R^d$ and $d(x,y)=\|x-y\|_2$.
\end{definition}

\begin{definition}
Given a distribution $q$ and a compact set $\mathcal{C}$, we denote by $q|_\mathcal{C}$ the distribution $q$ truncated to the set $\mathcal{C}$. The truncated measure is defined as $q|_\mathcal{C}(A) = \frac{1}{q(\mathcal{C})}q(A \cap \mathcal{C})$.
\end{definition}

\subsection{Main result} 

Our main result states that we can approximate any log-concave distribution in Wasserstein-1 distance by a \emph{well-conditioned} affine-coupling flow network. Precisely, we show:

\begin{theorem}
    \label{thm:main_universal_approx}
    Let $p(x): \R^d \to \R^+$ be of the form $p(x) \propto e^{-U(x)}$, such that:
    \begin{enumerate}[nosep]
        \item $U \in C^2$, i.e., $\nabla^2 U(x)$ exists and is continuous.
        \item $\ln p$ satisfies $\id_{d} \preceq - \nabla^2 \ln p(x) \preceq \kappa \id_{d}$. 
    \end{enumerate}
    Furthermore, let $p_0 := p \times \mathcal{N}(0,\id_d)$. Then, for every $\epsilon > 0$, there exists a compact set $\mathcal{C} \sub \R^{2d}$ and an invertible affine-coupling network $f: \R^{2d} \to \R^{2d}$ with polynomial non-linearities, such that 
    \[ W_1(f_\# (\mathcal{N}(0,\id_{2d}) \vert_\mathcal{C}), p_0 ) \le \epsilon.\] 
    Furthermore, the map defined by this affine-coupling network $f$ is well conditioned over $\mathcal C$, that is, there are positive constants $A(\kappa), B(\kappa) =  \ka^{O(1)}$ 
    such that for any unit vector $w$,
    \[ A(\ka) \le  \norm{D_w f(x,v)} \le B(\ka) \]
    for all $(x,v) \in \mathcal{C}$, where $D_w$ is the directional derivative in the direction $w$. In particular, the condition number of $Df(x,v)$ is bounded by $\frac{B(\ka)}{A(\ka)}=\ka^{O(1)}$ for all $(x,v) \in \mathcal{C}$.
\end{theorem}

We make several remarks regarding the statement of the theorem:

\begin{remark} The Gaussian padding (i.e. setting $p_0 = p \times \mathcal{N}(0,\id_d)$) is essential for our proofs. All the other prior works on the universal approximation properties of normalizing flows (with or without padding) result in ill-conditioned affine coupling networks. This gives theoretical backing of empirical observations on the benefits of Gaussian padding in \cite{koehler2020representational}.  
\end{remark}

\begin{remark} The choice of non-linearities $s,t$ being polynomials is for the sake of convenience in our proofs. Using standard universal approximation results, they can also be chosen to be neural networks with a smooth activation function.
\end{remark}

\begin{remark} The Jacobian $D f$ has both upper-bounded largest singular value, and lower-bounded smallest singular value---which of course bounds the determinant $\mbox{det}(D f)$. As remarked in Section~\ref{s:notation}, merely bounding the ratio of the two quantities would not suffice for this. Moreover, the bound we prove \emph{only} depends on properties of the distribution (i.e., $\kappa$), and does not worsen as $\epsilon \to 0$, in contrast to \cite{koehler2020representational}. 
\end{remark}

\begin{remark} The region $\mathcal{C}$ where the pushforward of the Gaussian through $f$  and $p_0$ are close is introduced solely for technical reasons---essentially, standard results in analysis for approximating smooth functions by polynomials can only be used if the approximation needs to hold on a compact set. Note that $\mathcal{C}$ can be made arbitrarily large by making $\epsilon$ arbitrarily small. 
\end{remark}

\begin{remark} We do not provide a bound on the number of affine coupling blocks, although a bound can be extracted from our proofs.
\end{remark}


\section{Preliminaries} 

Our techniques leverage tools from stochastic differential equations and dynamical systems. We briefly survey the relevant results. 

\subsection{Langevin Dynamics} 
\label{sec:langevin_dynamics}
Broadly, Langevin diffusions are families of stochastic differential equations (SDEs) which are most frequently used as algorithmic tools for sampling from distributions specified up to a constant of proportionality. They have also recently received a lot of attention as tools for designing generative models \citep{song2019generative, song2020score}. 

In this paper, we will only make use of \emph{underdamped Langevin dynamics}, a momentum-like analogue of the more familiar \emph{overdamped Langevin dynamics}, defined below. Our construction will involve simulating underdamped Langevin dynamics using affine coupling blocks.

\begin{definition}[Underdamped Langevin Dynamics] \emph{Underdamped Langevin dynamics} with potential $U$ and parameters $\ze$, $\ga$ is the pair of SDEs
\begin{align}\label{e:m-uld}
\begin{cases}
    dx_t &= -\zeta v_t dt\\
    dv_t &= -\ga \zeta v_t dt- \nabla U(x_t) dt + \sqrt{2 \gamma}\, dB_t.    
    \end{cases}
\end{align}
The stationary distribution of the SDEs (limiting distribution as $t \to \infty$) is given by $p^*(x,v) \propto e^{-U(x) - \fc{\zeta}2 \norm{v}^2}$.
\end{definition}

The variable $v_t$ can be viewed as a ``velocity'' variable and $x_t$ as a ``position'' variable -- in that sense, the above SDE is an analogue to momentum methods in optimization. 

The convergence of \eqref{e:m-uld} can be bounded when the distribution $p(x) \propto \exp(-U(x))$ satisfies an analytic condition, namely has a bounded \emph{log-Sobolev} constant.
Though we don't use the log-Sobolev constant in any substantive manner in this paper, we include the definition for completeness. 
\begin{definition}
A distribution $p:\R^d \to \R^+$ satisfies a \emph{log-Sobolev} inequality with constant $C > 0$ if $\forall g:\mathbb{R}^d \to \mathbb{R}$, s.t. $g^2, g^2 |\log g^2| \in L^1(p)$, we have  
\begin{equation}
    \mathbb{E}_p[g^2 \log g^2] - \mathbb{E}_p[g^2] \log \mathbb{E}_p[g^2] \leq 2C \mathbb{E}_p\|\nabla g\|^2 .
\end{equation}
\end{definition}
In the context of Markov diffusions (and in particular, designing sampling algorithms using diffusions), the interest in this quantity comes as it governs the convergence rate of \emph{overdamped} Langevin diffusion in the KL divergence sense. Namely, if $p_t$ is the distribution of overdamped Langevin after time $t$, one can show
\[\mbox{KL}(p_t || p) \leq e^{-Ct} \mbox{KL}(p_0 || p). \]

We will only need the following fact about the log-Sobolev constant: 
\begin{fact}[\cite{bakry1985diffusions,bakry2013analysis}] Let the distributions $p(x) \propto \exp(-U(x))$ be such that $U(x) \succeq \lambda I$. Then, $p$ has log-Sobolev constant bounded by $\lambda$. 
\end{fact} 

We will also need the following result characterizing the convergence time of \emph{underdamped} Langevin dynamics in terms of the log-Sobolev constant, as shown in \cite{ma2019mcmc}: 

\begin{theorem}[\cite{ma2019mcmc}] 
\label{thm:L-ma_mcmc}
Let $p^*(x) \propto \exp(-U(x))$ have a log-Sobolev constant bounded by $\rho$. 
Furthermore, for a distribution $p: \R^d \to \R^{+}$, let 

\[ \mathcal{L}[p] := \KL(p \Vert p^*) + \mathbb{E}_{p} \left[ \left\langle \nabla \frac{\delta \KL(p \Vert p^*)}{\delta p}, S \nabla \frac{\delta \KL(p \Vert p^*)}{\delta p} \right\rangle \right], \] where $S$ is a positive definite matrix given by
$S = \frac{1}{\kappa} \begin{bmatrix} \frac{1}{4} I_{d \times d} & \frac{1}{2} I_{d \times d}  \\ \frac{1}{2} I_{d \times d}  & 2 I_{d \times d} \end{bmatrix}$. 
If $p_t$ is the distribution of $(x_t, v_t)$ which evolve according to \eqref{e:m-uld}, we have 
\begin{equation}
\label{e:L-ma-mcmc}
    \frac{d}{dt} \mathcal{L}[p_t] \le - \frac{\rho}{10} \mathcal{L}[p_t] 
\end{equation}
whenever $p^*$ satisfies a log-Sobolev inequality with constant $\rho$.
\end{theorem}

We note that the above theorem uses a non-standard Lyapunov function $\mathcal{L}$, which combines KL divergence with an extra term, since the generator of underdamped Langevin is not self-adjoint---this makes analyzing the drop in $KL$ divergence difficult. As $\mathcal{L}$ is clearly an upper bound on $KL(p || p^*)$, so it suffices to show $\mathcal L$ decreases rapidly. 

We will also need a less-well-known \emph{deterministic} form of the updates which is equivalent to \eqref{e:m-uld}. Precisely, we convert \eqref{e:m-uld} an equivalent ODE (with time-dependent coefficients). The proof of this fact (via a straightforward comparison of the Fokker-Planck equation) can be found in \cite{ma2019mcmc}. 
\begin{theorem}
Let $p_t(x_t, v_t)$ be the probability distribution of running \eqref{e:m-uld} for time $t$. If started from $(x_0, v_0) \sim p_0$, the probability distribution of the solution $(x_t, v_t)$ to the ODEs 
\begin{align}\label{e:m-uld-3-0}
    \ddd t\coltwo{x_t}{v_t} 
    &=
    \matt{O}{I_d}{-I_d}{-\ga I_d} (\nb \ln p_t - \nb \ln p^*)
\end{align}
is also $p_t(x_t, v_t)$.
\end{theorem}

\subsection{Dynamical systems and Henon maps} 

We also build on work from dynamical systems, more precisely, a family of maps called \emph{H\'enon-like maps} \citep{henon1976two}.
\begin{definition}[\citep{turaev2002polynomial}]
A pair of ODEs forms a \emph{H\'enon-like map} if it has the form
\begin{equation}
\label{eqn:ode_henon}
\begin{cases}  
\frac{dx}{dt} = v \\ 
\frac{dv}{dt} = -x + \nabla J(x)
\end{cases}
\end{equation}
for a smooth function $J: \R^d \to \R$.
\end{definition}
This special family of ODEs is a continuous-time generalization of a classical discrete dynamical system of the same name \cite{henon1976two}. The property that is useful for us is that the Euler discretization of this map can be written as a sequence of affine coupling blocks.

In \cite{turaev2002polynomial}, it was proven that these ODEs are \emph{universal approximators} in some sense. Namely, the iterations of this ODE can approximate any \emph{symplectic diffeomorphism}: a continuous map which preserves volumes (i.e. the Jacobian of the map is 1). These kinds of diffeomorphisms have their genesis in Hamiltonian formulations of classical mechanics \citep{abraham2008foundations}. 

At first blush, symplectic diffeomorphisms and underdamped Langevin seem to have nothing to do with each other. The connection comes through the so-called Hamiltonian representation theorem \citep{polterovich2012geometry}, which states that any symplectic diffeomorphism from $\mathcal{C} \subseteq \R^{2d} \to \R^{2d}$ can be written as the iteration of the following \emph{Hamiltonian} system of ODEs for some (time-dependent) Hamiltonian $H(x,v,t)$:
\begin{equation}
\label{eqn:ode_hamiltonion}
\begin{cases}  
\frac{dx}{dt} = \frac{d}{dv} H(x,v,t) \\ 
\frac{dv}{dt} = -\frac{d}{dx} H(x,v,t)
\end{cases}
\end{equation}
In fact, in our theorem, we will use techniques inspired by those in  \cite{turaev2002polynomial}, who shows: 
\begin{theorem}[\cite{turaev2002polynomial}]
For any function $H(x,v, t):\mathbb{R}^{2d}\times \R_{\ge 0} \to \mathbb{R}$ which is polynomial in $(x,v)$, there exists a polynomial $V(x,v,t)$, s.t. the time-$\tau$ map of the system 
\begin{equation}
\begin{cases}  
\frac{dx}{dt} = \frac{\partial}{\partial v} H(x,v,t) \\ 
\frac{dv}{dt} = -\frac{\partial}{\partial x} H(x,v,t)
\end{cases}
\label{eq:turaev}
\end{equation}
is uniformly $O(\tau^2)$-close to the time-$2\pi$ map of the system  
\begin{equation}
\label{eq:turaev_approx}
\begin{cases}  
\frac{dx}{dt} = v \\ 
\frac{dv_j}{dt} = -\Omega_j^2 x_j - \tau  \frac{\partial}{\partial x_j} V(x,t)
\end{cases}
\end{equation}
for some integers $\{\Omega_i\}_{i=1}^d$. 
\end{theorem}

We will prove a generalization of this theorem that applies to underdamped Langevin dynamics.


\section{Proof Sketch of Theorem \ref{thm:main_universal_approx}}
\label{sec:proof_sketch}

\subsection{Overview of strategy}
\label{subsec:overview_strategy}



We wish to construct an affine coupling network that (approximately) pushes forward a Gaussian $p^*=\mathcal N(0,\id_{2d})$ to the distribution we wish to model with Gaussian padding, i.e. $p_0=p\times \mathcal N(0,\id_d)$. Because the inverse of an affine coupling network is an affine coupling network, we can invert the problem, and instead attempt to map $p_0$ to $N(0,\id_{2d})$.
\footnote{As an aside, a similar strategy is taken in practice by recent SDE-based generative models (\cite{song2020score}).}

There is a natural map that takes $p_0$ to $p^*=N(0,\id_{2d})$, namely, underdamped Langevin dynamics \eqref{e:m-uld}. Hence, our proof strategy involves understanding and simulating underdamped Langevin dynamics with the initial distribution $p_0 = p \times \mathcal{N}(0,\id_{d})$, and the target distribution $p^* = \mathcal{N}(0,\id_{2d})$, and comprises of two important steps.

First, we show that
    the flow-map for Langevin is well-conditioned (Lemma \ref{lemma:underdamped_conditioning} below). Here, by flow-map, we mean the map which assigns each $x$ to its evolution over a certain amount of time $t$ according to the equations specified by \eqref{e:m-uld}. 
    
    Second, we break the simulation of underdamped Langevin dynamics for a certain time $t$ into intervals of size $\tau$, and show that the \emph{inverse} flow-map over each $\tau$-sized interval of time can be approximated well by a composition of affine-coupling maps (Lemma \ref{l:ode_approx} below). To show this, we consider a more general system of ODEs than the one in \cite{turaev2002polynomial} (in particular, a non-Hamiltonian system), which can be applied to \emph{underdamped} Langevin dynamics. We then show that the \emph{inverse} flow-map of this system of ODEs can be approximated by a sequence of affine-coupling blocks. We note that for this argument, it is critical that we use underdamped rather than overdamped Langevin dynamics, as overdamped Langevin dynamics do not have the required form for affine-coupling blocks. 

\subsection{Underdamped Langevin is well-conditioned}
\label{subsec:underdamped_conditioning}
Consider running underdamped Langevin dynamics with stationary distribution $p^*$ equal to the standard Gaussian, started at a log-concave distribution with bounded condition number $\ka$. 
The following lemma says that the flow map is well-conditioned, with condition number depending polynomially on $\ka$.
\begin{lemma}
\label{lemma:underdamped_conditioning}
Consider underdamped Langevin dynamics~\eqref{e:m-uld}
with $\ze=1$, friction coefficient $\ga<2$ and starting distribution $p$ which satisfies all the assumptions in Theorem~\ref{thm:main_universal_approx}. Let $T_t$ denote the flow-map from time $0$ to time $t$ induced by~\eqref{e:m-uld-3-0}. Then for any $x_0,v_0 \in \R^d$ and unit vector $w$, the directional derivative of $T_t$ at $x_0,v_0$ in direction $w$ satisfies
\begin{align*}
    \pa{1+\fc{2+\ga}{2-\ga}(\ka-1)}^{-2/\ga} \le 
    \ve{D_w T_t(x_0)} \le \pa{1+\fc{2+\ga}{2-\ga}(\ka-1)}^{2/\ga}.
\end{align*}
Therefore, the condition number of $T_t$ is bounded by $\pa{1+\fc{2+\ga}{2-\ga}(\ka-1)}^{4/\ga}$.
\end{lemma}
We sketch the proof below and include a complete proof in Section~\ref{s:conditioning}.

First, using~\eqref{e:m-uld-3-0} and the chain rule shows that the Jacobian of the flow map at $x_0$, $D_t=DT_t(x_0)$, satisfies 
\begin{align}
    \label{e:jh-0}
    \ddd tD_t &=\matt{O}{I_d}{-I_d}{-\ga I_d} \nb^2 (\ln p_t - \ln p^*) D_t,
\end{align}
i.e., it is bounded by the difference of the Hessians of the log-pdfs of the current distribution and the stationary distribution. We will show that $\nb^2\ln p_t$ decays exponentially towards $\nb^2\ln p^* = \id_{2d}$.

To accomplish this, consider how $\nb^2 \ln p_t$ evolves if we replace~\eqref{e:m-uld} by its discretization, 
\begin{align*}
    \td x_{t+\eta} &= \td x_t + \eta \td v_t\\
    \td v_{t+\eta} &= (1-\eta \ga) \td v_{t} - \eta \td x_t + \xi_t,\quad \xi_t\sim N(0,2\eta \id_d).
\end{align*}
Note that because the stationary distribution is a Gaussian, $\nb U(x_t) = x_t$ in \eqref{e:m-uld}, and the above equations take a particularly simple form: we apply a linear transformation to $\coltwo{ \td x_t}{\td v_t}$, and then add Gaussian noise, which corresponds to convolving the current distribution by a Gaussian.
We keep track of upper and lower bounds for $\nb^2 \ln p_t$, and compute how they evolve under this linear transformation and convolution by a Gaussian. 
Taking $\eta\to 0$, we obtain differential equations for the upper and lower bounds for $\nb^2 \ln p_t$, which we can solve. A Gr\"onwall argument shows that these bounds decay exponentially towards $\nb^2\ln p^* = \id_{2d}$. The decay rate can be bounded as a power of $\rc{\ka}$.

From~\eqref{e:jh-0}, we then obtain that the condition number of $D_t$ is bounded by the integral of a exponentially decaying function, and hence is bounded independent of $t$. In particular, we may take $t$ large enough so that $p_t$ is $\ep$-close to the stationary distribution. Because the decay rate of the exponential is $\rc{\ka^{O(1)}}$, the bound is $\ka^{O(1)}$. 

Note that we vitally used the fact that the stationary distribution $p$ is a standard Gaussian, as our argument requires that $\nb^2 \ln p^*$ be constant everywhere.

\subsection{ODE approximation by affine-coupling blocks}

Next, we analyze a more general version of the Hamiltonian system of ODEs considered in \cite{turaev2002polynomial}, which we recalled in \eqref{eq:turaev}. In particular, the system of ODEs we will be considering is:
\begin{equation} 
\begin{cases}  
\frac{dx}{dt} = \frac{\partial}{\partial v} H(x,v,t) \\ 
\frac{dv}{dt} = -\frac{\partial}{\partial x} H(x,v,t) - \gamma \frac{\partial}{\partial v} H(x,v,t)
\end{cases}
\label{eq:underdamped_hamiltonian}
\end{equation}

Note that substituting $H(x,v,t) = \ln p_t(x,v) - \ln p^*(x,v)$ above gives us the underdamped Langevin dynamics.

The first step is to restrict our considerations to $H$ being a polynomial in $x,v$, rather than a general smooth function.
Towards this, we recall the notion of closeness in the $C^1$ topology:
\begin{definition}
Let $\mathcal{C} \subseteq \R^d$ be a compact set. Let $f,g : \mathcal{C} \to \R$ be two continuously differentiable functions. Then we say that $f,g$ are uniformly $\epsilon$-close over $\mathcal{C}$ in $C^1$ topology if
\[ \sup_{x \in \mathcal{C}} \brac{\norm{f(x) - g(x)} + \norm{Df(x) - Dg(x)}} \le \epsilon \]
\end{definition}

The following lemma (a generalization of the Stone-Weierstrass Theorem) then establishes that it suffices to focus on $H$ being polynomial in $x,v$:

\begin{lemma}[Theorem 5, \cite{peet2007exponentially}]
\label{lemma:poly_approx_smooth_fn}
    Let $\mathcal{C} \subset \R^{d}$ be a compact set. For any $C^2$ function $H:\R^d \to \R$, and any $\epsilon > 0$, there is a multivariate polynomial $P:\R^d \to \R$ such that $P,H$ are uniformly $\epsilon$-close over $\mathcal{C}$ in $C^1$ topology.
\end{lemma}
    
Focusing on the case of polynomials, 
Lemma \ref{lemma:poly_ode_approx} below shows that instead of flowing the pair of ODEs given by \eqref{eq:underdamped_hamiltonian} over an interval of time $\tau$, we can instead run a different ODE for time $2\pi$, such that the flow-maps corresponding to both these ODEs are $O(\tau^2)$-close. 
    
\begin{lemma}
\label{lemma:poly_ode_approx}
Let $\mathcal{C} \subset \R^{2d}$ be a compact set. For any function $H(x,v, t):\mathbb{R}^{2d} \to \mathbb{R}$ which is polynomial in $(x,v)$, there exist polynomial functions $J, F, G$, s.t. the time-$(t_0+\tau,t_0)$ flow map of the system
\begin{equation} 
\begin{cases}  
\frac{dx}{dt} = \frac{\partial}{\partial v} H(x,v,t) \\ 
\frac{dv}{dt} = -\frac{\partial}{\partial x} H(x,v,t) - \gamma \frac{\partial}{\partial v} H(x,v,t)
\end{cases}
\label{eqn:hamiltonian}
\end{equation}
is uniformly $O(\tau^2)$-close over $\mathcal{C}$ in $C^1$
topology to the time-$2\pi$ map of the system
\begin{equation}
\label{eqn:affine_coupling_ode}
\begin{cases}  
\frac{dx}{dt} = v - \tau F(v,t) \odot x \\ 
\frac{dv_j}{dt} = -\Omega_j^2 x_j - \tau J_j(x,t) - \tau v_j G_j(x,t)
\end{cases}
\end{equation}
Here, $\odot$ denotes component-wise product, and the constants inside the $O(\cdot)$ depend on $\mathcal{C}$ and the coefficients of $H.$
\end{lemma}

The complete proof of this lemma is included in Appendix \ref{sec:ode_approx_full}; we provide a brief sketch here. First, we consider the first order ($O(\tau^2)$) approximation  of the flow map of a standard ODE of the form $\dot{y}=Dy$ (where $D$ is diagonal), and observe that for small $\tau$, we can think of \eqref{eqn:affine_coupling_ode} as a perturbed version of such an ODE with an appropriate choice of $D$.
Using standard ODE perturbation techniques, we can approximately express the time-$t$ evolution of \eqref{eqn:affine_coupling_ode} up to first-order in $\tau$, in terms of polynomials $F,G,J$ and trigonometric functions.

Then, we compare this map to the first-order approximation of flowing the pair of ODEs \eqref{eqn:hamiltonian} for time $\tau$ via Taylor's theorem. Furthermore, this approximation is a polynomial in $(x,v)$ since $H$ is a polynomial in $(x,v)$.

The crucial step involves choosing the functional form of $F(z,t), J(z,t), G(z,t)$ suitably, so that they are polynomials in $z$ with coefficients in terms of $\sin(\Omega t)$, $\cos(\Omega t)$. After simplification, both expressions can be expressed in terms of polynomials in $x,v$ where coefficients can be expressed in terms of $\int_{0}^{2\pi}\sin^{p}(\Omega s)\cos^{q}(\Omega s)\,ds$, which either integrate to 0 or a constant. 
Thus, to ensure that the two approximations match, we are left with a problem of making two multivariate polynomials in $(x,v)$ equal. 

This final step can of course be written as a linear system of equations. We identify a special structure in this system, which helps us show that the system is full-rank, and hence has a solution. \hfill \qedsymbol


Finally, consider discretizing the newly constructed ODE \eqref{eqn:affine_coupling_ode} into small steps of size $\eta$ by a simple Euler schema i.e.,
\begin{equation}
\label{eqn:discretized_affine_coupling_ode}
\begin{cases}  
x_{n+1}= x_n +  \eta(v_n - \tau F(v_n,\eta n) \odot x_n) \\ 
v_{n+1, j} = v_{n,j} -\eta(\Omega_j^2 x_{n,j} - \tau J_j(x_n,\eta n) - \tau v_{n,j} G_j(x_n,\eta n))
\end{cases}
\end{equation}
We note that each step above can be written as a composition of two affine coupling blocks given by
$(x_n, v_n) \mapsto (x_n, v_{n+1}) \mapsto (x_{n+1},v_{n+1})$.
Namely, the map $(x_n, v_n) \mapsto (x_n, v_{n+1})$ can be written as
\begin{align*}
\begin{cases}
    x_n = x_n \\
    v_{n+1} = v_n \odot (1-\tau)G(x_n, \eta n) - \eta(\Omega^2 \odot x_n - \tau J(x_n, \eta n))
\end{cases}
\end{align*}
This map is an affine coupling block with $s(x_n) = (1-\tau) \odot G(x_n, \eta n)$ and $t(x_n) = - \eta(\Omega^2 \odot x_n - \tau J(x_n, \eta n))$.
The map $(x_n, v_{n+1}) \mapsto (x_{n+1}, v_{n+1})$ can be written as
\begin{align*}
\begin{cases}
    v_{n+1} = v_{n+1} \\
    x_{n+1}= x_n +  \eta(v_{n+1} - \tau F(v_{n+1},\eta n) \odot x_n) \\ 
\end{cases}
\end{align*}
which is an affine coupling block with $s(v_{n+1}) = 1-\eta \tau F(v_{n+1}, \eta n)$ and $t(v_{n+1}) = \eta v_{n+1}$.

The composition of the two maps above yields an affine coupling network $(x_n, v_n) \mapsto (x_{n+1}, v_{n+1})$ precisely as given by \Cref{eqn:discretized_affine_coupling_ode} with non-linearities $s,t$ in each of the blocks given by polynomials. The following lemma bounds the error resulting from this discretization:
\begin{lemma}[Euler's discretization method]
    \footnote{This result is well known in the $C^0$ topology, we provide an analysis for the $C^1$ bound in \Cref{s:discretization_error_proof}.}
    \label{lemma:discretization_error}
    Let $\mathcal{C} \subset \R^{2d}$ be a compact set. Consider discretizing the time from $0$ to $t$ into $\frac{t}{\eta}$ steps and performing the update given by \eqref{eqn:discretized_affine_coupling_ode} at each of these steps. Let the map obtained as a result of discretizing thus be denoted by $T'_t$ and let the original flow map be denoted by $T_t$. Then $T_t$ and $T'_t$ are uniformly $O(\eta)$ close over $\mathcal{C}$ in $C^1$ topology, and the constants inside the $O(\cdot)$ depend on $\mathcal{C}$, and bounds on the derivatives of $T_t$ over $\mathcal{C}$.
\end{lemma}

\subsection{Simulating by breaking into $\tau$-sized intervals}

Let $T_{s,t}$ denote the time-$s,t$ flow-map of \eqref{eq:underdamped_hamiltonian} from time $s$ to time $t$. Since the flow maps are invertible, $T_{s,t}$ and $T_{t,s}$ are inverses. 
We are now ready to state the following lemma which says that the underdamped Langevin flow-map $T_{\timethree,0}$ can be written as a composition of affine-couplings maps:
\begin{lemma}
\label{l:ode_approx}
Let $\mathcal{C} \subset \R^{2d}$ be a compact set. Suppose that $T_{\timethree,0}(x,v)$ is the time-$(\timethree,0)$ flow-map of the ODE's
\begin{equation}
\label{eqn:ode_approx}
\begin{cases}  
\frac{dx}{dt} = \frac{\partial}{\partial v} H(x,v,t) \\ 
\frac{dv}{dt} = -\frac{\partial}{\partial x} H(x,v,t) - \gamma \frac{\partial}{\partial v} H(x,v,t)
\end{cases}
\end{equation}
where $H$ is $C^\infty$. Then for any $\epsilon_1,  \timethree \in \R_+$, there exists an integer $N = N(\epsilon_1, \timethree, \mathcal{C})$ 
and affine-coupling blocks $f_1, \ldots, f_N$ such that the composition $f = f_N \circ \cdots \circ f_1$ is $\epsilon_1$-close to $T_{\timethree,0}$ in the $C^1$ topology over $\mathcal{C}$.
\end{lemma}
The proof of Lemma \ref{l:ode_approx} is in \Cref{sec:proof_ode_approx}. We provide a brief sketch here: from Lemma \ref{lemma:poly_approx_smooth_fn}, we know that it suffices to show the result for a polynomial $H$. Thereafter, we break the time for which we want to flow the ODE given by \eqref{eqn:ode_approx} into small chunks of length $\tau$. \Cref{lemma:poly_ode_approx,,lemma:discretization_error} then show that the flow map over this chunk can be written as an affine coupling network. Composing the affine coupling networks over all the chunks of time gives us the result.

\subsection{Putting components together}
The previous sections established that for any $t$ and any compact set $\mathcal{C}$, there is a affine-coupling network $f$ with polynomial non-linearities such that $T_{t,0}$ and $f$ are uniformly close over $\mathcal{C}$. We will now pick an appropriate value of $t$ and set $\mathcal{C}$ such that $W_1(f_\# (p^* |_{\mathcal{C}}), p_0) \le \epsilon$ where $p^* = \mathcal{N}(0,\id_{2d})$, which is the required result of \Cref{thm:main_universal_approx}.
First, using \Cref{thm:L-ma_mcmc}, for 
\[ \phi > -10 \log \epsilon_1 + \log 2 + \log \mathcal{L}[p_0]\]
we have that $\KL(T_{0,\phi\#} (p_0), p^*) \le \frac{\epsilon_1^2}{2}$. We use the following \emph{transportion cost inequality} to convert this to a Wasserstein bound.
\begin{theorem}[\citet{talagrand1996transportation}]
    The standard Gaussian $p$ on $\R^d$ satisfies a \emph{transportation cost inequality}: For every distribution $q$ on $\R^d$ with finite second moment, 
    $W_1(p,q)^2 \le 2 KL(q \Vert p)$.
\end{theorem}
This gives us that $W_1(T_{0,\phi \, \#} (p_0), p^*) \le \epsilon_1$. A simple argument in \Cref{l:wasserstein_bound_1} (\Cref{s:wasserstein_bounds}) then gives
\begin{equation}
\label{eqn:pit_triangle_ineq_1}
W_1(p_0, T_{\phi,0 \, \#} (p^*)) = W_1(T_{\phi,0 \, \#} (T_{0,\phi \, \#} (p_0)), T_{\phi,0 \, \#} (p^*)) \le \Lip(T_{\phi,0}) \epsilon_1
\end{equation}

A subsequent argument stated as \Cref{l:wasserstein_bound_2} in \Cref{s:wasserstein_bounds}, shows that if $f$ and $T_{\phi,0}$ are uniformly $\epsilon_1$-close in $C^0$ topology on some $\mathcal{C}$, then their pushforwards through $p^*|_\mathcal{C}$ are indeed close, i.e.,
\begin{equation}
\label{eqn:pit_triangle_ineq_2}
W_1(T_{\phi,0 \, \#} (p^* |_\mathcal{C}), f_\# (p^* |_\mathcal{C})) \le \epsilon_1.
\end{equation}

Next, we establish a bound on the Wasserstein distance between the standard Gaussian and its truncation on a compact set, proved in \Cref{s:conditional_wasserstein_proof}.

\begin{lemma}
\label{l:conditional_wasserstein}
Let $p^* = \mathcal{N}(0,\id_{2d})$. Then for every $\delta \in \R_+$, there exists a compact set $\mathcal{C} = B(0,R)$ such that $W_1(p^*, p^* |_\mathcal{C}) \le \delta$, where $B(0,R)$ denotes the ball of radius $R$ centered at the origin.
\end{lemma}
We now choose a compact set $\mathcal{C}$ such that \Cref{l:conditional_wasserstein} holds for $\delta=\epsilon_1$. Then \Cref{l:wasserstein_bound_1} again implies that
\begin{equation}
\label{eqn:pit_triangle_ineq_3}
W_1(T_{\phi,0 \, \#} (p^*), T_{\phi,0 \, \#}\ (p^* |_\mathcal{C})) \le \Lip(T_{\phi,0}) \epsilon_1 \end{equation}
Equations \eqref{eqn:pit_triangle_ineq_1}, \eqref{eqn:pit_triangle_ineq_2}, \eqref{eqn:pit_triangle_ineq_3} and the triangle inequality together imply 
\[ W_1(f_\#(p^* |_\mathcal{C}), p_0) \le (2 \Lip(T_{\phi,0}) + 1) \epsilon_1 \le \epsilon\]
for small enough $\ep_1$. We can indeed set $\epsilon_1$ small enough so as to satisfy the last inequality above, because of the global bound  $\Lip(T_{\phi,0}) \le \brac{1 + \frac{2 + \gamma}{2 - \gamma}(\kappa - 1)}^{2/\gamma}$ established in \Cref{lemma:underdamped_conditioning}. This gives us the statement of Theorem \ref{thm:main_universal_approx}. Note that the final value of $\phi$ depends on $\epsilon, \kappa, \gamma$ and $\mathcal{L}[p_0]$.

\ifdefined\PROOFS
Consider the coupling $\gamma'$, where a sample $(x,y) \sim \gamma'$ is generated as follows: first, we sample $z \sim p_0|_\mathcal{C}$, and then compute $x=f(z)$, $y=T_t(z)$. By definition of the pushforward, the marginals of $x$ and $y$ are $f_\#(p_0|_\mathcal{C})$ and ${T_t}_\#(p_0|_\mathcal{C})$ respectively. However, we know that for this $\gamma'$, from Lemma \ref{l:ode_approx}, $\|x-y\|_2 \le \epsilon$ uniformly. Thus, we can conclude that
\begin{align}
    W_1(f_\#(p_0|_\mathcal{C}), {T_t}_\#(p_0|_\mathcal{C})) &\le \int_{\R^d \times \R^d}\|x-y\|_2 \,d\gamma'(x,y) \\
    &\le \int_{\R^d \times \R^d}\epsilon\, d\gamma'(x,y)
    =\epsilon
\end{align}
\fi

\section{Related Work}
\label{sec:related_work}

The landscape of normalizing flow models is rather rich. The inception of the ideas was in \cite{rezende2015variational} and \cite{dinh2014nice}, and in recent years, an immense amount of research has been dedicated to developing different architectures of normalizing flows. The focus of this paper are affine coupling flows, which were introduced in \cite{dinh2014nice}, introduced the idea of using pushforward maps with triangular Jacobians for computational efficiency. This was further developed in \cite{dinh2016density} and culminated in \cite{kingma2018glow}, who introduced 1x1 convolutions in the affine coupling framework to allow for ``trainable'' choices of partitions. We note, there have been variants of normalizing flows in which the Jacobian is non-triangular, e.g. \citep{grathwohl2018ffjord, dupont2019augmented, behrmann2018invertible}, but these models still don't scale beyond datasets the size of CIFAR-10. 

In terms of theoretical results, the most closely related works are \cite{huang2020augmented, zhang2020approximation, koehler2020representational}. 
The former two show universal approximation of affine couplings---albeit if the input is padded with zeros. 
This of course results in maps with singular Jacobians, which is why this strategy isn't used in practice. \cite{koehler2020representational} show universal approximation without padding---though their constructions results in a flow model with condition number $1/\epsilon$ to get approximation $\epsilon$ in the Wasserstein sense, regardless of how well-behaved the distribution to be approximated is. Furthemore, \cite{koehler2020representational} provide some empirical evidence that padding with iid Gaussians (as in our paper) is better than both zero padding (as in \cite{huang2020augmented, zhang2020approximation}) and no padding on small-scale data.

\section{Conclusion}

In this paper, we provide the first guarantees on universal approximation with \emph{well-conditioned} affine coupling networks. The conditioning of the network is crucial when the networks are trained using gradient-based optimization of the likelihood. Mathematically, we uncover connections between stochastic differential equations, dynamical systems and affine coupling flows. 
Our construction uses Gaussian padding, which lends support to the empirical observation that this strategy tends to result in better-conditioned flows \citep{koehler2020representational}. We leave it as an open problem to generalize beyond log-concave distributions.

\newpage
\nocite{*}
\bibliographystyle{plainnat}
\bibliography{references}

\newpage

\appendix

\section{Conditioning}

\label{s:conditioning}


We analyze the condition number of underdamped Langevin dynamics with potential $f(x) = \rc 2 \ve{x}^2$ and stationary distribution $p(x,v) = e^{-f(x)-\rc 2\ve{v}^2} = e^{-\rc 2(\ve{x}^2+\ve{v}^2)}$. Underdamped Langevin dynamics is given by the following SDE's, 
\begin{align}\label{e:uld-1}
    dx_t &= -v_t\\
\nonumber 
    dv_t &= -\ga v_t - \nb f(x_t) + \sqrt{2} dB_t\\
    &= -\ga v_t - x_t + \sqrt{2} dB_t.     \label{e:uld-2}
\end{align}
Given the distribution $p_0$ at time 0, the distribution $p_t$ at time $t$ is the same as that given by,
\begin{equation}
\label{eq:gradient_flow}
    \begin{bmatrix} \frac{dx}{dt} \\ \frac{dv}{dt} \end{bmatrix}
    = - \begin{bmatrix} 0 & -\id_d \\ \id_d & \gamma \id_d \end{bmatrix} \begin{bmatrix} \nabla_x \frac{\delta \KL(\mathbf{p}_t \Vert \mathbf{p}^*)}{\delta \mathbf{p}_t} \\ \nabla_v \frac{\delta \KL(\mathbf{p}_t \Vert \mathbf{p}^*)}{\delta \mathbf{p}_t} \end{bmatrix}   
\end{equation}
which simplifies to 
\begin{align}\label{e:uld-3}
    d\coltwo{x_t}{v_t} 
    &=
    \matt{O}{\id_d}{-\id_d}{-\ga \id_d} (\nb \ln p_t - \nb \ln p).
\end{align}

Our goal is to prove the following theorem.
\begin{theorem}\label{t:uld-cond}
Consider underdamped Langevin dynamics~\eqref{e:uld-1}--\eqref{e:uld-2} 
with friction coefficient $\ga<2$ and starting distribution $p_0$ that is $C^2$. Let $T_t$ denote the transport map from time 0 to time $t$ induced by~\eqref{e:uld-3}. Suppose that the initial distribution $p_0(x,v)$ is such that 
\begin{align*}
 \id_{2d}\preceq 
    -\nb^2 \ln p_0(x,v)\preceq \ka \id_{2d}.
\end{align*}
Then for any $x_0,v_0$ and unit vector $w$, the directional derivative of $T_t$ at $x_0,v_0$ in direction $w$ satisfies
\begin{align*}
    \pa{1+\fc{2+\ga}{2-\ga}(\ka-1)}^{-2/\ga} \le 
    \ve{D_w T_t(x_0)} \le \pa{1+\fc{2+\ga}{2-\ga}(\ka-1)}^{2/\ga}
\end{align*}
Thus the condition number of $T_t$ is bounded by $\pa{1+\fc{2+\ga}{2-\ga}(\ka-1)}^{4/\ga}$.
\end{theorem} 
We remark that the exponent is likely loose by a factor of 2, and that taking $\ga\to 2$ gives the best exponent; however, the case $\ga=2$ would require a separate calculation as the matrix appearing in the exponential is not diagonalizable. Note $\ga=2$ is the transition between when the dynamics exhibit underdamped and overdamped behavior. 

To prove the theorem, we first relate the Jacobian with the Hessian of the log-pdf. 
By Lemma~\ref{l:ode_derivative}, the Jacobian $D_t = DT_t(x_0)$ satisfies
\begin{align}\label{e:jh}
    \ddd tD_t &=\matt{O}{\id_d}{-\id_d}{-\ga \id_d} \nb^2 (\ln p_t - \ln p) D_t.
\end{align}
We will show that $\nb^2 (\ln p_t - \ln p)$ decays exponentially (Lemma~\ref{l:var-exp-decay}). First, we need the following bound for convolutions.

\subsection{Bounding the Hessian of the logarithm of a convolution}

\begin{lemma}\label{l:convolve-lc2}
Suppose that $p$ is a probability density function on $\R^d$ such that 
$\Si_1^{-1}\preceq -\nb^2\ln p \preceq \Si_2^{-1}$.
Let $q$ be the distribution of $N(0,\Si)$ (where $\Si$ is not necessarily full-rank). Then  
$$(\Si_1+\Si)^{-1}\preceq -\nb^2\ln (p*q)\preceq (\Si_2+\Si)^{-1}.$$
\end{lemma}

\begin{proof} 
The lower bound is a bound on the strong log-concavity parameter; see Theorem 3.7b in~\cite{saumard2014log}. 

For the upper bound, we 
first prove the lemma in the case that 
    $\Si$ is full rank.
    We have $(p*q)(x) = \int_{\R^d} p(u)q(x-u)\,dt$, so 
    \begin{align*}
    \nb^2 [\ln ((p*q)(x))] &= 
    \fc{\int_{\R^d} p(u) \nb^2 q(x-u)\,du}{\int_{\R^d} p(u)  q(x-u)\,du}
    - 
    \pf{\int_{\R^d} p(u) \nb q(x-u)\,du}{\int_{\R^d} p(u)  q(x-u)\,du}
    \pf{\int_{\R^d} p(u) \nb q(x-u)\,du}{\int_{\R^d} p(u)  q(x-u)\,du}^\top\\
    &= 
     \pf{\int_{\R^d}\Si^{-1} (x-u) p(u) q(x-u)\,du}{\int_{\R^d} p(u)  q(x-u)\,du}
    \pf{\int_{\R^d} (\Si^{-1}(x-u))^\top  p(u) q(x-u)\,du}{\int_{\R^d} p(u)  q(x-u)\,du} \\
    &\quad - 
    \fc{\int_{\R^d} (\Si^{-1} (x-u)(x-u)^\top\Si^{-1} - \Si^{-1}) p(u) q(x-u)\,du}{\int_{\R^d} p(u)  q(x-u)\,du}
\end{align*}
Let $\mu_x$ denote the distribution with density function $\rh(u) \propto p(u)q(x-u)$. 
Then 
\begin{align*}
-\nb^2 [\ln ((p*q)(x))] &= 
    [\E_{\mu_x} \Si^{-1}(u-x)][\E_{\mu_x} (\Si^{-1}(u-x))^\top]
    - [\E_{\mu_x} \Si^{-1}(u-x)(u-x)^\top\Si^{-1}] + \Si^{-1}\\
    &= 
    -\E_{\mu_x} [\Si^{-1}(u-\E u)(u-\E u)^\top \Si^{-1}] + \Si^{-1}.
\end{align*}
It suffices to show for any unit vector $v$, that 
\begin{align*}
    -v^\top \nb^2 [\ln ((p*q)(x))] v & = 
    -\E_{\mu_x} [\an{\Si^{-1} v, (u-\E u)}^2] + v^\top \Si^{-1}v
    \le v^\top (\Si_2+\Si)^{-1} v
\end{align*}
Note that $\mu_x$ satisfies
\begin{align*}
    -\nb^2 \ln \mu_x &\preceq \Si_2^{-1}+\Si^{-1},
\end{align*}
so $\mu_x$ can be written as the density of a Gaussian with variance $(\Si_2^{-1}+\Si^{-1})^{-1}$ multiplied by a log-convex function. By the Brascamp-Lieb moment inequality (Theorem 5.1 in~\cite{brascamp2002extensions})\footnote{Note that the sign is flipped in the theorem statement in the log-convex case.}, 
\begin{align*}
    \E_{\mu_x} [\an{ \Si^{-1}v, (u-\E u)}^2]
    &\ge
    \E_{u\sim N(0,(\Si_2^{-1}+\Si^{-1})^{-1})} [\an{\Si^{-1}v,u}^2]
    =
    v^\top\Si^{-1} (\Si_2^{-1}+\Si^{-1})^{-1}\Si^{-1} v.
\end{align*}
Hence 
\begin{align*}
    -v^\top \nb^2 [\ln ((p*q)(x))] v & \le 
    v^\top 
    \ba{-\Si^{-1} (\Si_2^{-1}+\Si^{-1})^{-1}\Si^{-1} + \Si^{-1}}
    v
\end{align*}
The conclusion then follows from
\begin{align*}
    -\Si^{-1} (\Si_2^{-1}+\Si^{-1})^{-1}\Si^{-1} + \Si^{-1}
    &=
    -(\Si \Si_2^{-1} \Si + \Si)^{-1} +\Si^{-1}\\
    &=(\Si \Si_2^{-1} \Si + \Si)^{-1} (\cancel{-\id_d} + \Si \Si_2^{-1} + \cancel{\id_d})\\
    &= (\Si + \Si_2)^{-1}.
\end{align*}

Now for the general case, take the limit as $\Si'\to \Si$ where $\Si'$ is full-rank. More precisely, let $\Si_t = \Si + tP$, where $P$ is projection onto $\text{Im}(\Si)^{\perp}$, and let $q_t$ be the density function for $N(0,\Si_t)$. Then we have
\begin{align*}
    \nb^2 [\ln ((p*q_t)(x))]
    &= \fc{\int_{\R^d}  \nb^2 p(x-u) q_t(u)\,du}{\int_{\R^d} p(x-u)  q_t(u)\,du}
    - 
    \pf{\int_{\R^d} \nb p(x-u) q_t(u)\,du}{\int_{\R^d} p(x-u)  q_t(u)\,du}
    \pf{\int_{\R^d} \nb p(x-u)  q_t(u)\,du}{\int_{\R^d} p(x-u)  q_t(u)\,du}^\top
\end{align*}
Examining the first term, we have
\begin{align*}
    \int_{\R^d}  \nb^2 p(x-u) q_t(u)\,du
    &= \int_{\text{Im}(\Si)}\int_{\text{Im}(P)}
    \nb^2 p(x-u-v) q_t(u+v) \,dv \,du
    \\
    &\to \int_{\text{Im}(\Si)} \nb^2 p(x-u) q_t(u) \,du\text{ as }t\to 0^+
\end{align*}
by the dominated convergence theorem. Similarly, the other integrals converge to their counterparts with $q(u)$. Therefore, $\nb^2 [\ln ((p*q_t)(x))]\to \nb^2 [\ln ((p*q)(x))]$ as $t\to 0^+$. Apply the lemma to the full-rank case; the RHS bound converges to the desired bound: $(\Si_2+\Si_t)^{-1}\to (\Si_2+\Si)^{-1}$.

\end{proof}

\subsection{Bounding the variance proxy for underdamped Langevin}

As it is useful to work with the matrices $\Si_1$ and $\Si_2$, we make the following definition.
\begin{definition}
Let $p$ be a probability density on $\R^d$. For a positive definite matrix $\Si_1$, if $\Si_1^{-1}\preceq -\nb^2 \ln p$, we say that $\Si_1$ is an \vocab{upper variance proxy} for $p$. 
For a positive definite matrix $\Si_2$, if $-\nb^2 \ln p\preceq \Si_2^{-1}$, we say $\Si_2$ is a \vocab{lower variance proxy} for $p$.
\end{definition}

\begin{lemma}\label{l:var-exp-decay}
Consider underdamped Langevin dynamics~\eqref{e:uld-1}--\eqref{e:uld-2} with 
with starting distribution $p_0(x,v)$ that is $C^2$. 
Suppose $p_0$ has lower (upper) variance proxy $\Si_0$. Then $p_t$ has lower (upper) variance proxy
\begin{align*}
    \Si_t &= \exp\ba{\pa{\matt{}{1}{-1}{-\ga}\ot \id_d}t} (\Si_0-\id_{2d}) \exp\ba{\pa{\matt{}{-1}{1}{-\ga}\ot \id_d}t} + \id_{2d}.
\end{align*}
\end{lemma}
\begin{proof}
We first consider discretized Lanegevin, given by
\begin{align*}
    \td x_{t+\eta} &= \td x_t + \eta \td v_t\\
    \td v_{t+\eta} &= (1-\eta \ga) \td v_{t} - \eta \td x_t + \xi_t,\quad \xi_t\sim N(0,2\eta \id_d)
\end{align*}
or in matrix form,
\begin{align*}
    \coltwo{\td x_{t+\eta}}{\td v_{t+\eta}} &= \matt{\id_d}{\eta \id_d}{-\eta \id_d}{(1-\eta \ga)\id_d} \coltwo{\td x_{t}}{\td v_{t}} + \xi_t, \quad \xi_t \sim N\pa{0,\matt OOO{2\eta \id_d}}.
\end{align*}
Fix $t$. Let $\td p_t^{(\eta)}$ be the distribution at time $t$ for discretized Langevin with step size $\eta$ (dividing $t$). By standard arguments, $\td p_t^{(\eta)} \to p_t$ as $\eta \to 0$, in the $C^2$ topology on any compact set. In particular, for any $x,v$, $\nb^2 \ln \td p_t^{(\eta)}(x,v) \to \nb^2 \ln p_t(x,v)$. Hence it suffices to bound $\nb^2 \ln p_t(x,v)$.

We write the proof for the upper variance proxy; the proof for the lower variance proxy differs only in the direction of the inequality. 
Suppose $-\ln \td p_t(x,v)\succeq\td \Si_t^{-1}$. 
Consider breaking the update into two steps, 
\begin{align*}
    \coltwo{\td x_{t+\eta}'}{\td v_{t+\eta}'} &= \matt{\id_d}{\eta \id_d}{-\eta \id_d}{(1-\eta \ga)\id_d} \coltwo{\td x_{t}}{\td v_{t}} \\
    \coltwo{\td x_{t+\eta}}{\td v_{t+\eta}} &= \coltwo{\td x_{t+\eta}'}{\td v_{t+\eta}'} + \xi_t, \quad \xi_t \sim N\pa{0,\matt OOO{2\eta \id_d}}.
\end{align*}
Let $\td p_{t+\eta}'(x,v)$ denote the distribution of $\coltwo{\td x_{t+\eta}'}{\td v_{t+\eta}'}$. Then 
\begin{align*}
    \td p_{t+\eta}'(x,v) &= \td p_t\pa{
    \matt{\id_d}{\eta \id_d}{-\eta \id_d}{(1-\eta \ga)\id_d}^{-1}\coltwo xv
    } 
\end{align*}
so 
\begin{align*}
    \td \Si_{t+\eta}' :&=
    \matt{\id_d}{\eta \id_d}{-\eta \id_d}{(1-\eta \ga)\id_d} \td \Si_t \matt{\id_d}{-\eta \id_d}{\eta \id_d}{(1-\eta \ga)\id_d}
\end{align*}
is an upper variance proxy for $\td p_{t+\eta}'$ and by Lemma~\ref{l:convolve-lc2}, 
\begin{align*}
    \td \Si_{t+\eta} :&= \td \Si_{t+\eta}' + \matt OOO{2\eta \id_d}
\end{align*}
is an upper variance proxy for $\td p_{t+\eta}$. Note that 
\begin{align*}
    \td \Si_{t+\eta} :&= \td \Si_t
    +  \ba{\matt{}{1}{-1}{-\ga\eta }\ot \id_d} \td S_t
    + \td S_t \ba{\matt{}{-1}{1}{-\ga\eta }\ot \id_d}
    +  \matt 000{2\ga\eta} + O(\eta^2).
\end{align*}
By the standard analysis of Euler's method, as $\eta\to 0$, the distribution, $\td \Si_t$ approaches $\Si_t$ defined by 
\begin{align*}
    \ddd t \Si_t = \ba{\matt{}{1}{-1}{-\ga}\ot \id_d} \Si_t 
+ \Si_t \ba{\matt{}{-1}{1}{-\ga}\ot \id_d} + \matt 000{2\ga}.
\end{align*}
This $\Si_t$ is an upper variance proxy for $p_t$. The solution to this equation is
\begin{align*}
    \Si_t &= \exp\ba{\pa{\matt{}{1}{-1}{-\ga}\ot \id_d}t} (\Si_0-\id_{2d}) \exp\ba{\pa{\matt{}{-1}{1}{-\ga}\ot \id_d}t} + \id_{2d},
\end{align*}
as desired.
\end{proof}


\subsection{Proof that underdamped Langevin is well-conditioned}

We are now ready to prove the main theorem.
\begin{proof}[Proof of Theorem~\ref{t:uld-cond}]
Let $H_t = \nb^2(-\ln p_t + \ln p)$ and  $C=\matt{O}{\id_d}{-\id_d}{-\ga \id_d}$. 
By~\eqref{e:jh} and the chain rule,
\begin{align}\label{e:dDD}
    \ddd t D_tD_t^\top 
    &= -(C H_t D_tD_t^\top + D_tD_t^\top H_t C^\top).
\end{align}
Fix $w$ and consider $y_t = D_tw = D_wT_t(x_0)$. Multiplying the above by $W$ on both sides gives\footnote{The condition number bound in Theorem~\ref{t:uld-cond} is the square of what one might expect because we are only able to get obtain a bound on the absolute value here. If this is always increasing or decreasing, then we would save a factor of 2 in the exponent.}
\begin{align*}
    \ab{\ddd t \ve{y_t}^2} &\le 
    2\ve{C H_t} \ve{y_t}^2
\end{align*}
so by Gr\"onwall's inequality (Lemma~\ref{l:gronwall}),
\begin{align}\label{e:y-gw}
\exp\ba{-2\int_0^t \ve{CH_s}\,ds}\le 
    \ve{y_t}^2 &\le \exp\ba{2\int_0^t \ve{CH_s}\,ds}.
\end{align}
By Lemma~\ref{l:var-exp-decay},
\begin{align*}
    \id_{2d}\preceq -\nb^2\ln p_t
    \preceq
    (\ka-1)
\exp\ba{\pa{\matt{}{1}{-1}{-\ga}\ot \id_d}t} \exp\ba{\pa{\matt{}{-1}{1}{-\ga}\ot \id_d}t} + \id_{2d}.
\end{align*}
The eigenvalues of $A:=\matt{}{-1}{1}{-\ga}$ are $\fc{-\ga \pm \sqrt{\ga^2-4}}2$, which have absolute value 1.
The absolute value of the inner product of the eigenvectors of $A$ is $\ga/2$, so the condition number squared of the two exponential factors is bounded by 
$
    \fc{1+\fc\ga2}{1-\fc\ga2}
    = \fc{2+\ga}{2-\ga}.
$
In full detail, we calculate
\begin{align*}
    \exp\pa{\matt{}{-1}{1}{\ga} t}
    &= 
    \ub{\matt{1}{1}{\fc{\ga-\sqrt{\ga^2-4}}2}{\fc{\ga+\sqrt{\ga^2-4}}2}}{S}
    \ub{\matt{\exp\pa{\fc{-\ga+\sqrt{\ga^2-4}}{2}t}}{}{}{\exp\pa{\fc{-\ga-\sqrt{\ga^2-4}}{2}t}}}{D}\\
    &\quad \cdot 
    \ub{\rc{\sqrt{\ga^2-4}} 
    \matt{\fc{\ga+\sqrt{\ga^2-4}}2}{-1}{\fc{-\ga+\sqrt{\ga^2-4}}2}{1}
     }{S^{-1}}\\
     \ve{S^\dagger S}&=
     \ve{\matt{2}{\fc{\ga^2+\ga \sqrt{\ga^2-4}}2}{\fc{\ga^2-\ga \sqrt{\ga^2-4}}2}2}=2+\ga\\
     \ve{\exp\pa{\matt{}{-1}{1}{\ga} t}}
     &\le \fc{2+\ga}{\sqrt{4-\ga^2}} \exp\pf{-\ga t}2 = \sfc{2+\ga}{2-\ga} \exp\pf{-\ga t}2.
\end{align*}
Hence $H_t = -\nb^2\ln p_t + \id_{2d}$ satisfies 
\begin{align*}
    \ve{CH_s} &\le 
    1-\fc{1}{1+
    \fc{2+\ga}{2-\ga}(\ka-1) e^{-\ga t/2} }\\
    \int_0^\iy \ve{CH_s} \,ds
    &\le 
    \int_0^\iy 
    \fc{\fc{2+\ga}{2-\ga}(\ka-1) e^{-\ga t/2}}{1+
    \fc{2+\ga}{2-\ga}(\ka-1) e^{-\ga t/2} }\,ds
    \\
    &\le 
    \ba{\fc{2}{\ga}
\ln \pa{1+\fc{2+\ga}{2-\ga}(\ka-1)e^{-\ga t/2}}}^0_\iy
\le \fc{2}{\ga}\ln \pa{1+\fc{2+\ga}{2-\ga}(\ka-1)}.
\end{align*}
Hence by~\eqref{e:y-gw},
\begin{align*}
    \pa{1+\fc{2+\ga}{2-\ga}(\ka-1)}^{-2/\ga}
\le 
    \ve{y_t} &\le \pa{1+\fc{2+\ga}{2-\ga}(\ka-1)}^{2/\ga},
\end{align*}
giving the theorem. To obtain the bound on condition number, note that the condition number of $DT_t(x_0)$ is $\fc{\max_{\ve{w}=1} \ve{D_w T_t(x_0)}}{\min_{\ve{w}=1} \ve{D_w T_t(x_0)}}$.
\end{proof}



\section{Proof of \Cref{lemma:poly_ode_approx}}
\label{sec:ode_approx_full}
For the sake of convenience, we restate \Cref{lemma:poly_ode_approx} again.
\begin{lemma*}
Let $\mathcal{C} \in \R^{2d}$ be a compact set. For any function $H(x,v, t):\mathbb{R}^{2d}\times \R_{\ge 0} \to \mathbb{R}$ which is polynomial in $(x,v)$, there exist polynomial functions $J$, $F$, $G$, s.t. the time-$(t_0 + \tau, t_0)$ flow map of the system 
\begin{equation} 
\begin{cases}  
\frac{dx}{dt} = \frac{\partial}{\partial v} H(x,v,t) \\ 
\frac{dv}{dt} = -\frac{\partial}{\partial x} H(x,v,t) - \gamma \frac{\partial}{\partial v} H(x,v,t)
\end{cases}
\label{eq:hamiltonian}
\end{equation}
is uniformly $O(\tau^2)$-close over $\mathcal{C}$ in $C^1$ topology to the time-$2\pi$ map of the system  
\begin{equation}
\label{eq:ham_approx}
\begin{cases}  
\frac{dx}{dt} = v - \tau F(v,t) \odot x \\ 
\frac{dv_j}{dt} = -\Omega_j^2 x_j - \tau J_j(x,t) - \tau v_j G_j(x,t)
\end{cases}
\end{equation}
for some integers $\{\Omega_j\}_{j=1}^d$. Here, $\odot$ denotes component-wise product, and the constants inside the $O(\cdot)$ depend on $\mathcal{C}$ and the coefficients of $H$.
\end{lemma*}
\begin{proof}
First, note that the time-$(t_0 + \tau, t_0)$ flow map of \eqref{eq:hamiltonian} is equal to the time-$(t_0, t_0+\tau)$ flow map of the system:
\begin{equation} 
\begin{cases}  
\frac{dx}{dt} = -\frac{\partial}{\partial v} H(x,v,t_0 + \tau-t) \\ 
\frac{dv}{dt} = \frac{\partial}{\partial x} H(x,v,t_0 + \tau-t) + \gamma \frac{\partial}{\partial v} H(x,v,t_0 + \tau-t)
\end{cases}
\label{eq:hamiltonian_reverse}
\end{equation}
Proceeding ahead, we broadly follow the proof strategy in \cite{turaev2002polynomial}.
For notational convenience, let's denote the initial vector by $x(0),v(0)$ (each coordinate is specified separately). Let 
\begin{align} 
\label{e:x0}
x^0_j(t) &= x_j(0) \cos \Omega_j t + \frac{1}{\Omega_j} v_j(0) \sin \Omega_j t \\
\label{e:v0}
v^0_j(t) &= - \Omega_j x_j(0) \sin \Omega_j t + v_j(0) \cos \Omega_j t .
\end{align}
Using perturbative ODE techniques (see \cref{s:perturbation_analysis}), the solution to \eqref{eq:ham_approx} satisfies 
\begin{equation}\small
\label{e:perturbed_sol}
\begin{cases}  
x(t) = x^0(t) - \tau \int_{0}^{t}\left( \rc{\Om}\odot J(x^0(s),s) \odot \sin\Omega (t-s)  +   F(v^0(s),s) \odot \cos \Omega (t-s) \odot x^0(s) \right.\\
\hspace{3cm} +  \left.\rc{\Om} \odot G(x^0(s),s) \odot \sin \Omega (t-s) \odot v^0(s) \right) ds + O(\tau^2)\\  
v(t) = v^0(t) - \tau \int_{0}^{t}\left( J(x^0(s),s) \odot \cos\Omega (t-s)  -  \Om \odot F(v^0(s),s) \odot \sin \Omega (t-s) \odot x^0(s) \right. \\
\hspace{3cm}\left. + G(x^0(s),s) \odot \cos \Omega (t-s) \odot v^0(s)  \right) ds + O(\tau^2)  
\end{cases}
\end{equation}
Substituting $t = 2 \pi$, the time-$2 \pi$ map of \eqref{eq:ham_approx} is given by
\begin{equation}\small
\begin{cases}  
x(2\pi) = x^0(2\pi) - \tau \int_{0}^{2\pi}\left( -\rc{\Om}\odot J(x^0(s),s) \odot \sin\Omega s  +   F(v^0(s),s) \odot \cos \Omega s \odot x^0(s) \right.\\
\hspace{3cm}\left.- \rc{\Om}\odot G(x^0(s),s) \odot \sin \Omega s \odot v^0(s) \right) ds + O(\tau^2)\\  
v(2\pi) = v^0(2\pi) - \tau \int_{0}^{2\pi}\left( J(x^0(s),s) \odot \cos\Omega s  + \Om \odot F(v^0(s),s) \odot \sin \Omega s \odot x^0(s) \right.\\
\hspace{3cm}\left. + G(x^0(s),s) \odot \cos \Omega s \odot v^0(s)  \right) ds + O(\tau^2)\\  
\end{cases}
\end{equation}
Note that this holds if $\Omega$ is integral, and we will choose it to be so.

On the other hand, using Taylor's theorem, the solution to \eqref{eq:hamiltonian} satisfies: 
\begin{equation} 
\begin{cases}  
x(\tau) = x(0) - \tau \frac{\partial}{\partial v} H(x(0),v(0), t_0 + \tau)  + O(\tau^2)\\ 
v(\tau) = v(0) + \tau \frac{\partial}{\partial x} H(x(0),v(0), t_0 + \tau) 
+ \tau \gamma \frac{\partial}{\partial v} H(x(0),v(0), t_0 + \tau)
+ O(\tau^2)
\end{cases}
\end{equation}
We will now show that for any two polynomials $r_1,r_2$ of total degree at most $M$ we can choose functions $J,F,G$, s.t.:
\begin{equation} 
\begin{cases}  
\int_{0}^{2\pi}\left( -\rc{\Om}\odot J(x^0(s),s) \odot \sin\Omega s  +   F(v^0(s),s) \odot \cos \Omega s \odot x^0(s) \right.\\
\hspace{3cm}\left. - \rc{\Om}\odot G(x^0(s),s) \odot \sin \Omega s \odot v^0(s) \right) ds = r_1(x(0), y(0))\\  
\int_{0}^{2\pi} \left( J(x^0(s),s) \odot \cos\Omega s  + \Om \odot F(v^0(s),s) \odot \sin \Omega s \odot x^0(s) \right.\\
\hspace{3cm}\left.+ G(x^0(s),s) \odot \cos \Omega s \odot v^0(s)  \right) ds = r_2(x(0), y(0)) \\
\end{cases}
\label{eq:mainpair}
\end{equation}

We will choose $J,F,G$ of the form: 
\begin{equation}
\label{e:fgj_defintions}
\hspace{-1cm}
\begin{cases}
\forall j \in [d]: J_j(z,t) = \sum_{\mathbf{i}:|\mathbf{i}| \leq M} v_{j,\mathbf{i}}^J(t) z^\mathbf{i} \\
\forall j \in [d]: F_j(z,t) = \sum_{\mathbf{i}:|\mathbf{i}| \leq M-1} v_{j,\mathbf{i}}^F(t) z^\mathbf{i} \\    
\forall j \in [d]: G_j(z,t) = \sum_{\mathbf{i}:|\mathbf{i}| \leq M-1} v_{j,\mathbf{i}}^G(t) z^\mathbf{i}
\end{cases}
\end{equation} 
where $\mathbf{i} = (i_1, \ldots, i_d)$ denotes multi-index, and $\abs{\mathbf{i}} = \sum_{k=1}^d i_k$ and $z^{\mathbf{i}} = \prod_{k=1}^d z_k^{i_k}$.
Let 
\begin{align}\label{e:r1} r_{1,j}(x(0),v(0)) &= \sum_{\mathbf{k}: \abs{\mathbf{k}} \leq M} \sum_{\mathbf{p} + \mathbf{q} = \mathbf{k}} h^1_{j,\mathbf{p},\mathbf{q}} x(0)^\mathbf{p} v(0)^\mathbf{q}\\
\label{e:r2}
r_{2,j}(x(0),v(0)) &= \sum_{\mathbf{k}: \abs{\mathbf{k}} \leq M} \sum_{\mathbf{p} + \mathbf{q} = \mathbf{k}} h^2_{j,\mathbf{p},\mathbf{q}} x(0)^\mathbf{p} v(0)^\mathbf{q} \end{align}
The equation \eqref{eq:mainpair} gives us that for all $j$,
\begin{equation} 
\begin{cases}  
\int_{0}^{2\pi}\left( -\rc{\Om_j} J_j(x^0(s),s) \sin (\Omega_j s)  +   F_j(v^0(s),s) \cos (\Omega_j s) x^0_j(s) \right.\\
\hspace{3cm}\left.- \rc{\Om_j} G_j(x^0(s),s) \sin (\Omega_j s) v^0_j(s) \right) ds = r_{1,j}(x(0), y(0))\\  
\int_{0}^{2\pi} \left( J_j(x^0(s),s) \cos(\Omega_j s)  + \Om_j F_j(v^0(s),s) \sin (\Omega_j s) x^0_j(s) \right.\\
\hspace{3cm}\left. + G_j(x^0(s),s) \cos (\Omega_j s) v^0_j(s)  \right) ds = r_{2,j}(x(0), y(0)) 
\end{cases}
\label{e:r-equation}
\end{equation}
Let $\binom{\mathbf{k}}{\mathbf{p}} = \prod_{k = 1}^d \binom{k_i}{p_i}$. Let $\mathbf{k^t_j}$ be the multi-index $(k_1, \ldots, k_j + t, \ldots, k_d)$. 
We substitute~\eqref{e:x0}--\eqref{e:v0}, ~\eqref{e:fgj_defintions}, and~\eqref{e:r1}--\eqref{e:r2} into~\eqref{e:r-equation} and match the coefficients of $x(0)^\mathbf{p} v(0)^\mathbf{q}$.

If $k_j = 0$, then
\begin{align*}
	h^1_{j,\mathbf{p},\mathbf{q}}
	& = \int_0^{2\pi} - \rc{\Omega_j} v^J_{j,\mathbf{k}} \cos(\Omega s)^\mathbf{p} \sin(\Omega s)^\mathbf{q_j^1} \binom{\mathbf{k}}{\mathbf{p}} ds\\
	h^2_{j,\mathbf{p},\mathbf{q}}
	& = \int_0^{2\pi} v^J_{j,\mathbf{k}} \cos(\Omega s)^\mathbf{p_j^1} \sin(\Omega s)^\mathbf{q} \binom{\mathbf{k}}{\mathbf{p}} ds
\end{align*}
where $v^J_{j,\mathbf{k}} = a \cos(\Omega s)^\mathbf{p} \sin(\Omega s)^\mathbf{q_j^1} + b \cos(\Omega s)^\mathbf{p_j^1} \sin(\Omega s)^\mathbf{q}$. Since the function $\delta(s) = \cos(\Omega s)^{\mathbf{p} + \mathbf{p_j^1}} \sin(\Omega s)^{\mathbf{q} + \mathbf{q_j^1}}$ satisfies $\delta(\pi - s) = - \delta(\pi + s)$, this function integrates to zero, and hence the system above reduces to 
\begin{align*}
	h^1_{j,\mathbf{p},\mathbf{q}} & = a \rc{\Omega_j} C \binom{\mathbf{k}}{\mathbf{p}} \\
	h^2_{j,\mathbf{p},\mathbf{q}} & = b C \binom{\mathbf{k}}{\mathbf{p}}
\end{align*}
for some non-zero constant 
\[ C = \int_0^{2 \pi} \cos(\Omega s)^{2 \mathbf{p}} \sin(\Omega s)^{2 \mathbf{q_j^1}} ds = \int_0^{2 \pi} \cos(\Omega s)^{2 \mathbf{p_j^1}} \sin(\Omega s)^{2 \mathbf{q}} ds \]
Note that the integral is non-zero since the function inside is positive as all the powers are even.

If $k_j > 0$, then substituting the forms of $x^0(s), v^0(s)$ from \eqref{e:x0} in the LHS of ~\eqref{e:r-equation}, and expanding using the binomial theorem, we get that
\begin{align*}
	h^1_{j,\mathbf{p}, \mathbf{q}}
	& =
	\rc{\Om^{\mathbf{q^1_j}}}
	\int_0^{2 \pi} - v^J_{j,\mathbf{k}} \cos(\Omega s)^\mathbf{p}  \sin(\Omega s)^{\mathbf{q^1_j}} \binom{\mathbf{k}}{\mathbf{p}} ds\\
	& + \Om^{\mathbf{p^{-1}_j}} \int_0^{2 \pi} v^F_{j,\mathbf{k^{-1}_j}} (-\mathbf{1})^\mathbf{p^{-1}_j} \sin(\Omega s)^\mathbf{p^{-1}_j} \cos(\Omega s)^\mathbf{q^2_j} \binom{\mathbf{k^{-1}_j}}{\mathbf{p^{-1}_j}}ds \\
	&+ \Om^{\mathbf{p^{-1}_j}} \int_{0}^{2\pi}v^F_{j,\mathbf{k^{-1}_j}} (-\mathbf{1})^\mathbf{p} \sin(\Omega s)^\mathbf{p^1_j} \cos(\Omega s)^\mathbf{q} \binom{\mathbf{k^{-1}_j}}{\mathbf{p}} ds\\
	& +\rc{\Om^{\mathbf{q}}} \int_0^{2 \pi} \left(v^G_{j, \mathbf{k^{-1}_j}} \cos(\Omega s)^\mathbf{p^{-1}_j} \sin(\Omega s)^\mathbf{q^2_j} \binom{\mathbf{k^{-1}_j}}{\mathbf{p^{-1}_j}}
	- v^G_{j, \mathbf{k^{-1}_j}} \cos(\Omega s)^\mathbf{p^1_j} \sin(\Omega s)^\mathbf{q} \binom{\mathbf{k^{-1}_j}}{\mathbf{p}}\right)ds\\
	h^2_{j,\mathbf{p}, \mathbf{q}}
	& = \rc{\Om^{\mathbf{q}}}\int_0^{2 \pi} v^J_{j,\mathbf{k}} \cos(\Omega s)^\mathbf{p^1_j} \sin(\Omega s)^{\mathbf{q}} \binom{\mathbf{k}}{\mathbf{p}} ds\\
	& + \Om^{\mathbf{p}}\int_0^{2 \pi} v^F_{j,\mathbf{k^{-1}_j}} (-\mathbf{1})^\mathbf{p^{-1}_j} \sin(\Omega s)^\mathbf{p} \cos(\Omega s)^\mathbf{q^1_j} \binom{\mathbf{k^{-1}_j}}{\mathbf{p^{-1}_j}}ds
	\\&+ \Om^{\mathbf{p}}\int_0^{2 \pi}v^F_{j,\mathbf{k^{-1}_j}} (-\mathbf{1})^\mathbf{p} \sin(\Omega s)^\mathbf{p^2_j} \cos(\Omega s)^\mathbf{q^{-1}_j} \binom{\mathbf{k^{-1}_j}}{\mathbf{p}} ds\\
	& +\rc{\Om^{\mathbf{q^{-1}_j}}} \int_0^{2 \pi}\left( - v^G_{j, \mathbf{k^{-1}_j}} \cos(\Omega s)^\mathbf{p} \sin(\Omega s)^\mathbf{q^1_j} \binom{\mathbf{k^{-1}_j}}{\mathbf{p^{-1}_j}}
	+ v^G_{j, \mathbf{k^{-1}_j}} \cos(\Omega s)^\mathbf{p^2_j} \sin(\Omega s)^\mathbf{q^{-1}_j} \binom{\mathbf{k^{-1}_j}}{\mathbf{p}}\right)ds
\end{align*}
Let $g_{\mathbf{k}, \mathbf{p}}(s) = \cos(\Omega s)^\mathbf{p} \sin(\Omega s)^\mathbf{k - p}$ for all $\mathbf{p} \le \mathbf{k}$. Crucially, let us assume that $v^J_{j,\mathbf{k}}, v^F_{j,\mathbf{k}}, v^G_{j,\mathbf{k}}$ are all of the form
\begin{equation}
\begin{cases}
v^F_{j,\mathbf{k}} = \sum_{\mathbf{r} \le \mathbf{k^2_j}} \alpha_{\mathbf{k^2_j},\mathbf{r}} g_{\mathbf{k^2_j},\mathbf{r}}(s)\\
v^G_{j,\mathbf{k}} = \sum_{\mathbf{r} \le \mathbf{k^2_j}} \beta_{\mathbf{k^2_j},\mathbf{r}} g_{\mathbf{k^2_j},\mathbf{r}}(s)\\
v^J_{j,\mathbf{k}} = \sum_{\mathbf{r} \le \mathbf{k^1_j}} \gamma_{\mathbf{k^1_j},\mathbf{r}} g_{\mathbf{k^1_j},\mathbf{r}}(s)\\
\end{cases}
\end{equation}
Substituting,
\begin{align*}
	h^1_{j,\mathbf{p}, \mathbf{q}}
	& = \rc{\Om^{\mathbf{q^1_j}}}\int_0^{2 \pi} 
	- \sum_{\mathbf{r} \le \mathbf{k^1_j}} \gamma_{\mathbf{k^1_j},\mathbf{r}} g_{\mathbf{k^1_j},\mathbf{r}}(s)
	g_{\mathbf{k^1_j},\mathbf{p}} (s)
	\binom{\mathbf{k}}{\mathbf{p}} ds\\
	& + \Om^{\mathbf{p^{-1}_j}} \int_0^{2 \pi}\left( 
	(-\mathbf{1})^\mathbf{p^{-1}_j}
	\sum_{\mathbf{r} \le \mathbf{k^1_j}} \alpha_{\mathbf{k^1_j},\mathbf{r}} g_{\mathbf{k^1_j},\mathbf{r}}(s)
	g_{\mathbf{k^1_j},\mathbf{q^2_j}}(s)
	\binom{\mathbf{k^{-1}_j}}{\mathbf{p^{-1}_j}}
	+ (-\mathbf{1})^\mathbf{p}
	\sum_{\mathbf{r} \le \mathbf{k^1_j}} \alpha_{\mathbf{k^1_j},\mathbf{r}} g_{\mathbf{k^1_j},\mathbf{r}}(s)
	g_{\mathbf{k^1_j},\mathbf{q}}(s) 
	\binom{\mathbf{k^{-1}_j}}{\mathbf{p}}\right) ds\\
	& + \rc{\Om^{\mathbf{q}}}\int_0^{2 \pi}\left( 
	\sum_{\mathbf{r} \le \mathbf{k^1_j}} \beta_{\mathbf{k^1_j},\mathbf{r}} g_{\mathbf{k^1_j},\mathbf{r}}(s)
	g_{\mathbf{k^1_j},\mathbf{p^{-1}_j}}(s)
	\binom{\mathbf{k^{-1}_j}}{\mathbf{p^{-1}_j}}
	-  \sum_{\mathbf{r} \le \mathbf{k^1_j}} \beta_{\mathbf{k^1_j},\mathbf{r}} g_{\mathbf{k^1_j},\mathbf{r}}(s)
	g_{\mathbf{k^1_j},\mathbf{p^1_j}}(s)
	\binom{\mathbf{k^{-1}_j}}{\mathbf{p}}\right) ds\\
	h^2_{j,\mathbf{p}, \mathbf{q}}
	& = \rc{\Om^{\mathbf{q}}}\int_0^{2 \pi}
	\sum_{\mathbf{r} \le \mathbf{k^1_j}} \gamma_{\mathbf{k^1_j},\mathbf{r}} g_{\mathbf{k^1_j},\mathbf{r}}(s)
	g_{\mathbf{k^1_j}, \mathbf{p^1_j}}(s)
	\binom{\mathbf{k}}{\mathbf{p}} ds\\
	& + \Om^{\mathbf{p}}\int_0^{2 \pi}\left( 
	(-\mathbf{1})^\mathbf{p^{-1}_j}
	\sum_{\mathbf{r} \le \mathbf{k^1_j}} \alpha_{\mathbf{k^1_j},\mathbf{r}} g_{\mathbf{k^1_j},\mathbf{r}}(s)
	g_{\mathbf{k^1_j},\mathbf{q^1_j}}(s)
	\binom{\mathbf{k^{-1}_j}}{\mathbf{p^{-1}_j}}
	+ (-\mathbf{1})^\mathbf{p}
	\sum_{\mathbf{r} \le \mathbf{k^1_j}} \alpha_{\mathbf{k^1_j},\mathbf{r}} g_{\mathbf{k^1_j},\mathbf{r}}(s) 
	g_{\mathbf{k^1_j},\mathbf{q^{-1}_j}}(s)
	\binom{\mathbf{k^{-1}_j}}{\mathbf{p}}\right) ds\\
	& +\rc{\Om^{\mathbf{q^{-1}_j}}} \int_0^{2 \pi}\left(
	- \sum_{\mathbf{r} \le \mathbf{k^1_j}} \beta_{\mathbf{k^1_j},\mathbf{r}} g_{\mathbf{k^1_j},\mathbf{r}}(s)
	g_{\mathbf{k^1_j},\mathbf{p}}(s)
	\binom{\mathbf{k^{-1}_j}}{\mathbf{p^{-1}_j}}
	+ \sum_{\mathbf{r} \le \mathbf{k^1_j}} \beta_{\mathbf{k^1_j},\mathbf{r}} g_{\mathbf{k^1_j},\mathbf{r}}(s)
	g_{\mathbf{k^1_j},\mathbf{p^2_j}}(s)
	\binom{\mathbf{k^{-1}_j}}{\mathbf{p}}\right) ds
\end{align*}
Now, let $\inprod{f}{g} = \int_0^{2 \pi} f(s)g(s) ds$ denote the $\ell_2$ inner product. Then, we can rewrite the above system as
\begin{align*}
	h^1_{j,\mathbf{p}, \mathbf{q}}
	& =
	-\rc{\Om^{\mathbf{q^1_j}}} \sum_{\mathbf{r} \le \mathbf{k^1_j}} \gamma_{\mathbf{k^1_j},\mathbf{r}}
	\inprod{g_{\mathbf{k^1_j},\mathbf{r}}(s)}{g_{\mathbf{k^1_j},\mathbf{p}} (s)}
	\binom{\mathbf{k}}{\mathbf{p}} \\
	& + \Om^{\mathbf{p^{-1}_j}}\ba{
	(-\mathbf{1})^\mathbf{p^{-1}_j}
	\sum_{\mathbf{r} \le \mathbf{k^1_j}} \alpha_{\mathbf{k^1_j},\mathbf{r}}
	\inprod{g_{\mathbf{k^1_j},\mathbf{r}}(s)}{g_{\mathbf{k^1_j},\mathbf{q^2_j}}(s)}
	\binom{\mathbf{k^{-1}_j}}{\mathbf{p^{-1}_j}}
	+ (-\mathbf{1})^\mathbf{p}
	\sum_{\mathbf{r} \le \mathbf{k^1_j}} \alpha_{\mathbf{k^1_j},\mathbf{r}}
	\inprod{g_{\mathbf{k^1_j},\mathbf{r}}(s)}{g_{\mathbf{k^1_j},\mathbf{q}}(s)}
	\binom{\mathbf{k^{-1}_j}}{\mathbf{p}}} \\
	& + \rc{\Om^{\mathbf{q}}}
	\ba{\sum_{\mathbf{r} \le \mathbf{k^1_j}} \beta_{\mathbf{k^1_j},\mathbf{r}}
	\inprod{g_{\mathbf{k^1_j},\mathbf{r}}(s)}{g_{\mathbf{k^1_j},\mathbf{p^{-1}_j}}(s)}
	\binom{\mathbf{k^{-1}_j}}{\mathbf{p^{-1}_j}}
	-  \sum_{\mathbf{r} \le \mathbf{k^1_j}} \beta_{\mathbf{k^1_j},\mathbf{r}}
	\inprod{g_{\mathbf{k^1_j},\mathbf{r}}(s)}{g_{\mathbf{k^1_j},\mathbf{p^1_j}}(s)}
	\binom{\mathbf{k^{-1}_j}}{\mathbf{p}}}\\
	h^2_{j,\mathbf{p}, \mathbf{q}}
	& =\rc{\Om^{\mathbf{q}}}
	\sum_{\mathbf{r} \le \mathbf{k^1_j}} \gamma_{\mathbf{k^1_j},\mathbf{r}}
	\inprod{g_{\mathbf{k^1_j},\mathbf{r}}(s)}{g_{\mathbf{k^1_j}, \mathbf{p^1_j}}(s)}
	\binom{\mathbf{k}}{\mathbf{p}}\\
	& +\Om^{\mathbf{p}}\ba{
	(-\mathbf{1})^\mathbf{p^{-1}_j}
	\sum_{\mathbf{r} \le \mathbf{k^1_j}} \alpha_{\mathbf{k^1_j},\mathbf{r}}
	\inprod{g_{\mathbf{k^1_j},\mathbf{r}}(s)}{g_{\mathbf{k^1_j},\mathbf{q^1_j}}(s)}
	\binom{\mathbf{k^{-1}_j}}{\mathbf{p^{-1}_j}}
	+ (-\mathbf{1})^\mathbf{p}
	\sum_{\mathbf{r} \le \mathbf{k^1_j}} \alpha_{\mathbf{k^1_j},\mathbf{r}}
	\inprod{g_{\mathbf{k^1_j},\mathbf{r}}(s)}{g_{\mathbf{k^1_j},\mathbf{q^{-1}_j}}(s)}
	\binom{\mathbf{k^{-1}_j}}{\mathbf{p}}}\\
	& +\rc{\Om^{\mathbf{q^{-1}_j}}} 
	\ba{- \sum_{\mathbf{r} \le \mathbf{k^1_j}} \beta_{\mathbf{k^1_j},\mathbf{r}}
	\inprod{g_{\mathbf{k^1_j},\mathbf{r}}(s)}{g_{\mathbf{k^1_j},\mathbf{p}}(s)}
	\binom{\mathbf{k^{-1}_j}}{\mathbf{p^{-1}_j}}
	+ \sum_{\mathbf{r} \le \mathbf{k^1_j}} \beta_{\mathbf{k^1_j},\mathbf{r}}
	\inprod{g_{\mathbf{k^1_j},\mathbf{r}}(s)}{g_{\mathbf{k^1_j},\mathbf{p^2_j}}(s)}
	\binom{\mathbf{k^{-1}_j}}{\mathbf{p}}}
\end{align*}
Now, we will add a few redundant constraints in the system. These are added to ensure that the system has a nice matrix form; they are all of the type $0 = 0$. To do this, we allow $\mathbf{p} \ge \mathbf{0^{-1}_j}$, instead of $\mathbf{p} \ge \mathbf{0}$. Note that if $p_j = -1$, then $q_j = k_j + 1$  since $\mathbf{p} + \mathbf{q} = \mathbf{k}$. Again, we follow the convention that $\binom{n}{i} = 0$ if $i < 0$ or $i > n$, as well as $g_{\mathbf{k},\mathbf{p}} = 0$ if $\mathbf{p}$ is not between $\mathbf{0}$ and $\mathbf{k}$, both inclusive. Also define $h^1_{\mathbf{p}, \mathbf{q}} = h^2_{\mathbf{p},\mathbf{q}} = 0$ if either $\mathbf{p}$ or $\mathbf{q}$ are not between $\mathbf{0}$ and $\mathbf{k}$. 
Thus, all the new constraints added are indeed of the type $0 = 0$.

After these modifications, the system obtained has one constraint corresponding to $h^t_{\mathbf{p},\mathbf{q}}$ for each $\mathbf{0} \le \mathbf{q} \le \mathbf{k^1_j}$ (or equivalently $\mathbf{0^{-1}_j} \le \mathbf{p} \le \mathbf{k}$), $\mathbf{p} + \mathbf{q} = \mathbf{k}$, $t = 1,2$ with variables $\alpha_{\mathbf{k^1_j},\mathbf{r}}, \beta_{\mathbf{k^1_j},\mathbf{r}}, \gamma_{\mathbf{k^1_j},\mathbf{r}}$ for $\mathbf{0} \le \mathbf{r} \le \mathbf{k^1_j}$. Further, let
\[ n_{j, \mathbf{k}} = \abs{D_\mathbf{k}} \hspace{2cm} D_\mathbf{k} =  \set{\mathbf{r} : \mathbf{0} \le \mathbf{r} \le \mathbf{k}}\]

We will write this system in a matrix form, given by a matrix $A_{j, \mathbf{k}}$ of dimension $2 n_{j, \mathbf{k^1_j}} \times 3 n_{j, \mathbf{k^1_j}}$ such that
\[ A_{j,\mathbf{k}} \begin{bmatrix} \alpha \\ \beta \\ \gamma \end{bmatrix} = \begin{bmatrix} h^1_j \\ h^2_j \end{bmatrix} \]
Here $\xi = (\xi_{\mathbf{k^1_j}, \mathbf{r}})$ is the vector of dimension $n_{j,\mathbf{k^1_j}}$ for $\xi \in \set{\alpha, \beta, \gamma}$. For notational convenience, we will fix $j$ and $\mathbf{k}$ and denote $A = A_{j, \mathbf{k}}$. We will index rows of $A$ by $(\mathbf{p},t)$ and columns by $(\mathbf{r},\xi)$ where $\mathbf{r}, \mathbf{p^1_j} \in D_{\mathbf k_j^1}$, $t \in \set{1,2}$, $\xi \in \set{\alpha,\beta,\gamma}$. Further, we will denote by $A_{t,\xi}$ the submatrix of $A$ corresponding to the rows $(\mathbf{p}, t)$ and columns $(\mathbf{r}, \xi)$, that is, $A_{t,\xi}(\mathbf{p},\mathbf{r}) = A((\mathbf{p},t),(\mathbf{r},\xi))$. Matrix $A$ has only $2n_{j,\mathbf{k}}$ non-trivial rows, namely the rows which correspond to $\mathbf{p}$ such that $\mathbf{p} \ge 0$. Hence to show that the system above has a solution, it suffices to prove that matrix $A$ has rank $2n_{j,\mathbf{k}}$.

Define $X, Y$ to be $n_{j,\mathbf{k}} \times n_{j,\mathbf{k}}$ matrices with rows and columns indexed by elements of $D_\mathbf{k}$ such that 
\[ X(\mathbf{p}, \mathbf{r}) = \inprod{g_{\mathbf{k^1_j},\mathbf{r}}}{g_{\mathbf{k^1_j},\mathbf{p^1_j}}} \]
\[ Y(\mathbf{p}, \mathbf{r}) = (-\mathbf{1})^{\mathbf{p^1_j}} \inprod{g_{\mathbf{k^1_j},\mathbf{r}}}{g_{\mathbf{k^1_j},\mathbf{k^1_j} - \mathbf{p^1_j}}} \]
Now, assign $\Omega_1=1$, $\Omega_j = \frac{M^j-1}{M-1}$ for $j>1$. For this choice of $\Omega_j$'s, it is shown in \cite{turaev2002polynomial} that the functions $g_{\mathbf{k},\mathbf{s}}$ for $\mathbf{0} \le \mathbf{s} \le \mathbf{k}$ are linearly independent. It follows from this that the matrices $X$ and $Y$ are full rank.
Let $P$ be the permutation matrix that takes row $\mathbf{r}$ of this matrix to row $\mathbf{r^1_j}$ unless $r_j = k_j$, in which case it takes row $\mathbf{r}$ to $\mathbf{s}$ where $s_i = r_i$ for all $i \neq j$ and $s_j = -1$. Thus, for any matrix $M$, $PM(\mathbf{p}, \mathbf{r}) = M(\mathbf{p^{-1}_j},\mathbf{r})$ when $p_j \neq -1$, and $PM(\mathbf{p},\mathbf{r}) = M(\mathbf{p'},\mathbf{r})$ where $p'_i = p_i$ for $i \neq j$ and $p'_i = k_j$ if $p_j = -1$. 
In particular,
\[ PX(\mathbf{p},\mathbf{r}) = X(\mathbf{p^{-1}_j},\mathbf{r}) = \inprod{g_{\mathbf{k^1_j},\mathbf{r}}}{g_{\mathbf{k^1_j}, \mathbf{p}}} \]
\[ PY(\mathbf{p},\mathbf{r}) = Y(\mathbf{p^{-1}_j}, \mathbf{r}) = (-1)^\mathbf{p} \inprod{g_{\mathbf{k^1_j}, \mathbf{r}}}{g_{\mathbf{k_j^1},\mathbf{k^1_j} - \mathbf{p}}}\]
when $\mathbf{p} \ge \mathbf{0}$.
Define $n_{j,\mathbf{k}} \times n_{j,\mathbf{k}}$ diagonal matrices $D_1, D_2, D_3$ such that
\[ D_1(\mathbf{p}, \mathbf{p}) = \binom{\mathbf{k^{-1}_j}}{\mathbf{p}} \qquad D_2(\mathbf{p},\mathbf{p}) = \binom{\mathbf{k^{-1}_j}}{\mathbf{p^{-1}_j}} \qquad D_3(\mathbf{p},\mathbf{p}) = \binom{\mathbf{k}}{\mathbf{p}} \]
for $\mathbf{0^{-1}_j} \le \mathbf{p} \le \mathbf{k}$. Recalling that $\mathbf{q}=\mathbf{k}-\mathbf{p}$, we see that

\begingroup
\allowdisplaybreaks
\begin{align*}
	A_{1,\alpha}(\mathbf{p},\mathbf{r})
	& = \Omega^{\mathbf{p^{-1}_j}}\binom{\mathbf{k^{-1}_j}}{\mathbf{p^{-1}_j}} (-\mathbf{1})^\mathbf{p^{-1}_j} \inprod{g_{\mathbf{k^1_j,\mathbf{r}}}}{g_{\mathbf{k^1_j},\mathbf{k_j^1} - \mathbf{p^{-1}_j}}} + \Omega^{\mathbf{p^{-1}_j}}\binom{\mathbf{k^{-1}_j}}{\mathbf{p}} (-\mathbf{1})^\mathbf{p} \inprod{g_{\mathbf{k^1_j},\mathbf{r}}}{g_{\mathbf{k^1_j},\mathbf{k^1_j} - \mathbf{p^1_j}}} \\
	& = \Omega^{\mathbf{p^{-1}_j}}D_2(\mathbf{p},\mathbf{p}) P^2Y(\mathbf{p},\mathbf{r}) - \Omega^{\mathbf{p^{-1}_j}}D_1(\mathbf{p},\mathbf{p}) Y(\mathbf{p},\mathbf{r})\\
	\Rightarrow A_{1,\alpha}
	& = \Omega^{\mathbf{p^{-1}_j}}(D_2 P^2 - D_1) Y \\
	A_{1,\beta}(\mathbf{p},\mathbf{r})
	& = \rc{\Omega^{\mathbf{q}}}\binom{\mathbf{k^{-1}_j}}{\mathbf{p^{-1}_j}} \inprod{g_{\mathbf{k^1_j,\mathbf{r}}}}{g_{\mathbf{k^1_j},\mathbf{p^{-1}_j}}} - \rc{\Omega^{\mathbf{q}}}\binom{\mathbf{k^{-1}_j}}{\mathbf{p}} \inprod{g_{\mathbf{k^1_j},\mathbf{r}}}{g_{\mathbf{k^1_j},\mathbf{p^1_j}}} \\
	& = \rc{\Omega^{\mathbf{q}}}D_2(\mathbf{p},\mathbf{p}) P^2X(\mathbf{p},\mathbf{r}) - \rc{\Omega^{\mathbf{q}}}D_1(\mathbf{p},\mathbf{p}) X(\mathbf{p},\mathbf{r})\\
	\Rightarrow A_{1,\beta}
	& = \rc{\Omega^{\mathbf{q}}}(D_2 P^2 - D_1)X \\
	A_{1,\gamma}(\mathbf{p},\mathbf{r})
	& = - \frac{1}{\Omega^{\mathbf{q^1_j}}}\binom{\mathbf{k}}{\mathbf{p}} \inprod{g_{\mathbf{k^1_j},\mathbf{r}}}{g_{\mathbf{k^1_j},\mathbf{p}}}\\
	& = - \frac{1}{\Omega^{\mathbf{q^1_j}}}D_3(\mathbf{p},\mathbf{p}) PX(\mathbf{p},\mathbf{r}) \\
	\Rightarrow A_{1,\gamma}
	& = - \frac{1}{\Omega^{\mathbf{q^1_j}}}D_3 PX \\
	A_{2,\alpha}(\mathbf{p},\mathbf{r})
	& = \Omega^{\mathbf{p}}\binom{\mathbf{k^{-1}_j}}{\mathbf{p^{-1}_j}} (-\mathbf{1})^\mathbf{p^{-1}_j} \inprod{g_{\mathbf{k^1_j,\mathbf{r}}}}{g_{\mathbf{k^1_j},\mathbf{k_j^1} - \mathbf{p}}} + \Omega^{\mathbf{p}}\binom{\mathbf{k^{-1}_j}}{\mathbf{p}} (-\mathbf{1})^\mathbf{p} \inprod{g_{\mathbf{k^1_j},\mathbf{r}}}{g_{\mathbf{k^1_j},\mathbf{k^1_j} - \mathbf{p^2_j}}} \\
	& = - \Omega^{\mathbf{p}}D_2(\mathbf{p},\mathbf{p}) PY(\mathbf{p},\mathbf{r}) + \Omega^{\mathbf{p}}D_1(\mathbf{p},\mathbf{p}) P^{-1}Y(\mathbf{p},\mathbf{r})\\
	\Rightarrow A_{2,\alpha}
	& = \Omega^{\mathbf{p}}(-D_2 P + D_1 P^{-1}) Y \\
	A_{2,\beta}(\mathbf{p},\mathbf{r})
	& = - \rc{\Omega^{\mathbf{q^{-1}_j}}}\binom{\mathbf{k^{-1}_j}}{\mathbf{p^{-1}_j}} \inprod{g_{\mathbf{k^1_j,\mathbf{r}}}}{g_{\mathbf{k^1_j},\mathbf{p}}} + \rc{\Omega^{\mathbf{q^{-1}_j}}}\binom{\mathbf{k^{-1}_j}}{\mathbf{p}} \inprod{g_{\mathbf{k^1_j},\mathbf{r}}}{g_{\mathbf{k^1_j},\mathbf{p^2_j}}} \\
	& = - \rc{\Omega^{\mathbf{q^{-1}_j}}}D_2(\mathbf{p},\mathbf{p}) PX(\mathbf{p},\mathbf{r}) + \rc{\Omega^{\mathbf{q^{-1}_j}}}D_1(\mathbf{p},\mathbf{p}) P^{-1}X(\mathbf{p},\mathbf{r})\\
	\Rightarrow A_{2,\beta}
	& = \rc{\Omega^{\mathbf{q^{-1}_j}}}(-D_2 P + D_1 P^{-1})X \\
	A_{2,\gamma}(\mathbf{p},\mathbf{r})
	& = \rc{\Omega^{\mathbf{q}}}\binom{\mathbf{k}}{\mathbf{p}} \inprod{g_{\mathbf{k^1_j},\mathbf{r}}}{g_{\mathbf{k^1_j},\mathbf{p^1_j}}}\\
	& = \rc{\Omega^{\mathbf{q}}}D_3(\mathbf{p},\mathbf{p}) X(\mathbf{p},\mathbf{r}) \\
	\Rightarrow A_{2,\gamma}
	& = \rc{\Omega^{\mathbf{q}}}D_3 X
\end{align*}
\endgroup
For the above equations to go through as is, we need to check the case when $p_j = -1$, since definitions of $PX$ and $PY$ are different for this case. But, in this case, $D_1(\mathbf{p},\mathbf{p}) = D_2(\mathbf{p},\mathbf{p}) = 0$, and hence the equations hold. Similarly, we need to check the case $p_j = 0$ for blocks $A_{1,\alpha}$ and $A_{1,\beta}$, but again, $D_2(\mathbf{p},\mathbf{p}) = 0$ and hence the equations hold. Thus, we can write $A$ as 
\[ \begin{bmatrix}\id & 0 \\
0 & \Omega_j\id\end{bmatrix}
\begin{bmatrix} D_2 P^2 - D_1 & D_2 P^2 - D_1 & - D_3 P \\ -D_2 P + D_1 P^{-1} & -D_2 P + D_1 P^{-1} & D_3 \end{bmatrix}
\small\begin{bmatrix}
\Omega^{\mathbf{p^{-1}_j}}\id & 0 & 0 \\
0 & \rc{\Omega^{\mathbf{q}}}\id & 0 \\
0 & 0 & \frac{1}{\Omega^{\mathbf{q^1_j}}}\id
\end{bmatrix}
\begin{bmatrix} Y & 0 & 0 \\ 0 & X & 0 \\ 0 & 0 & X \end{bmatrix} \]
To show that $A$ has rank $2n_{j,\mathbf{k}}$, it suffices to show that the matrix

\[ B = \begin{bmatrix} D_2 P^2 - D_1 & - D_3 P \\ -D_2 P + D_1 P^{-1} & D_3 \end{bmatrix} \]
has rank $2n_{j,\mathbf{k}}$. Let us index rows of $B$ using $(\mathbf{p}, s)$ and columns using $(\mathbf{p},t)$ for $s,t \in \set{1,2}$. Since $P$ is a permutation matrix, post multiplying by $P$ takes column $\mathbf{r}$ of this matrix to column $\mathbf{r^{-1}_j}$, where the indices cycle whenever they are out of bounds.
More specifically,
\[ MP(\mathbf{p},\mathbf{r}) = P^{-1} M^\intercal (\mathbf{r},\mathbf{p}) = M^\intercal (\mathbf{r^1_j},\mathbf{p}) = M(\mathbf{p},\mathbf{r^1_j}) .\]
Hence, for a fixed row $(\mathbf{p},1)$ the non-zero entries in $B$ are in columns $(\mathbf{p^{-2}_j},1), (\mathbf{p},1), (\mathbf{p^{-1}_j},2)$. Similarly, non-zero entries in the row $(\mathbf{p},2)$ are in columns $(\mathbf{p^{-1}_j},1), (\mathbf{p^1_j},1), (\mathbf{p},2)$. Observe that rows $(\mathbf{p^1_j},1)$ and $(\mathbf{p},2)$ have non-zero entries in the same columns. This gives us a procedure to convert this matrix into a lower triangular matrix using row operations, where indices are ordered using any order $<_R$ that respects
\begin{enumerate}[nosep]
    \item $(\mathbf{p},t) <_R (\mathbf{q},t)$ if $p_j < q_j$
    \item $(\mathbf{p},1) <_R (\mathbf{q},2)$ for all $\mathbf{0^{-1}_j} \le \mathbf{p}, \mathbf{q} \le \mathbf{k}$
\end{enumerate}
In particular, any lexicographical ordering with highest priority to the $j^{th}$ coordinate works.

Note that only upper triangular non-zero entries using any such ordering are of the type $((\mathbf{p_j^1},1),(\mathbf{p},2))$. Now, we eliminate these using the following row operations:
\[ R(\mathbf{p_j^1},1) \leftarrow R(\mathbf{p_j^1},1) + C_\mathbf{p} R(\mathbf{p},2)) \]
for all $\mathbf{p}$ such that $0 \le \mathbf{p} \le \mathbf{k_j^{-1}}$. Here 
\[ C_\mathbf{p} = -\frac{B((\mathbf{p_j^1},1),(\mathbf{p},2))}{B((\mathbf{p},2),(\mathbf{p},2))} = - \frac{-\binom{\mathbf{k}}{\mathbf{p_j^1}}}{\binom{\mathbf{k}}{\mathbf{p}}} = \frac{\binom{k_j}{p_j+1}}{\binom{k_j}{p_j}} = \frac{k_j - p_j}{p_j + 1} \]
Note that after this set of operations, $B((\mathbf{p_j^1},1),(\mathbf{p},2)) \leftarrow 0$. On the other hand,
\begin{align*}
	B((\mathbf{p_j^1},1),(\mathbf{p_j^1},1))
	\leftarrow & B((\mathbf{p_j^1},1),(\mathbf{p_j^1},1)) + \frac{k_j - p_j}{p_j + 1} B((\mathbf{p},2), (\mathbf{p_j^1},1)) \\
	= & - \binom{\mathbf{k_j^{-1}}}{\mathbf{p_j^1}} + \frac{k_j - p_j}{p_j + 1} \binom{\mathbf{k_j^{-1}}}{\mathbf{p}}\\
	= & \binom{\mathbf{k_j^{-1}}}{\mathbf{p}} \brac{-\frac{k_j - p_j - 1}{p_j + 1} + \frac{k_j - p_j}{p_j + 1}} \\
	= & \frac{1}{p_j + 1} \binom{\mathbf{k_j^{-1}}}{\mathbf{p}} \neq 0
\end{align*}
The only non-zero entries in the upper triangle after this operation corresponds to positions $((\mathbf{p_j^1},1),(\mathbf{p},2))$, for $\mathbf{0_j^{-1}} \le \mathbf{p} \le \mathbf{k_j^{-1}}$, such that $p_j = -1$. To eliminate these, we perform the following row operations:
\[ R(\mathbf{p_j^1},1) \leftrightarrow R(\mathbf{p},2) \] 
for all $\mathbf{0_j^{-1}} \le \mathbf{p} \le \mathbf{k_j^{-1}}$ such that $p_j = -1$. Hence,
\[ B((\mathbf{p},2),(\mathbf{p},2)) \leftarrow B((\mathbf{p_j^1},1),(\mathbf{p},2)) = \binom{\mathbf{k}}{\mathbf{p_j^1}} \neq 0 \]
Note that $R(\mathbf{p},2) = 0$ since this row corresponds to a dummy constraint. Also, the other two non-zero entries in $R(\mathbf{p_j^1},1)$ are in the first half, and hence this does not create any upper triangular entries. Hence, this matrix is in fact lower triangular, in the given ordering $<_R$ of indices.

After the operations, among the diagonal terms, $B((\mathbf{p},2),(\mathbf{p},2)) \neq 0$ for $\mathbf{0_j^{-1}} \le \mathbf{p} \le \mathbf{k}$. Also, $B((\mathbf{p},1),(\mathbf{p},1)) \neq 0$ for $\mathbf{0_j^1} \le \mathbf{p} \le \mathbf{k}$. Therefore, the total number of non-zero diagonal entries is 
\[ n_{j,\mathbf{k}} \brac{\frac{k_j + 1}{k_j} + \frac{k_j-1}{k_j}} = 2 n_{j ,\mathbf{k}} \]
This proves that the matrix has rank $2n_{j,\mathbf{k}}$, which is the same as the number of non-trivial rows, and hence the system has a solution for any $r_1,r_2$. Consequently, we can always find polynomial functions $J,F,G$ as required.
\end{proof}

\section{Proof of Lemma \ref{l:ode_approx}}
\label{sec:proof_ode_approx}

\begin{proof}
    From \Cref{lemma:poly_approx_smooth_fn}, it suffices to focus on $H$ being a polynomial. We break the time from $\timethree$ to $0$ for which we want to flow the ODE given by \eqref{eqn:ode_approx} into $(n+1)$ small chunks of length $\tau$, i.e., let $\tau = \timethree / (n+1)$. Further, let $A_i = T_{(n-i+1)\tau,(n-i)\tau}$. Then, the time-$\timethree$ flow map can be write as the composition of $n+1$ maps, that is
    \[ T_{\timethree,0} = T_{\tau,0} \circ \cdots \circ T_{\timethree,\timethree-\tau} = A_n \circ \cdots \circ A_0 \]
    Let $\mathcal C_0=T_{0,\phi}(\mathcal C)$. Let $\mathcal C_1,\ldots, \mathcal C_{n+1}$ be a sequence of compact sets such that $A_i(\mathcal C_i)$ is in the interior of $\mathcal C_{i+1}$; by choosing them small enough, we can make $\mathcal C_{n+1}$ an arbitrary compact set containing $\mathcal C$ in its interior. Below, we treat $A_0,\ldots, A_n$ (and their approximations) as maps $\mathcal C_0\to \mathcal C_1\to \cdots \to \mathcal C_{n+1}$, and when we take the $C^1$ norm, we do it on the appropriate compact set. For small enough $\eta$, the $\eta$-discretized maps will stay inside the $\mathcal C_i$.
    
    Let $S_i$ denote the time-$2\pi$ flow map obtained by running the ODE system \eqref{eqn:affine_coupling_ode} from Lemma \ref{lemma:poly_ode_approx} above which approximates the map $T_{(n-i+1)\tau,(n-i)\tau} = A_i$. Further, let $S'_i$ denote the map obtained by discretizing the ODE system as in \eqref{eqn:discretized_affine_coupling_ode} with step size $\eta$. Then, we have that for each $i$, as $\eta \to 0$, 
    \begin{align*}
    \norm{S'_i - A_i}_{C^1}  &\le \|S'_i - S_i + S_i - A_i\|_{C^1} \\
    &\le \|S'_i - S_i\|_{C^1} + \|S_i - A_i\|_{C^1} \\
    &\le O(\eta) + O(\tau^2) \tag{by \Cref{lemma:discretization_error,,lemma:poly_ode_approx}}
    \end{align*}
    We choose $\eta = \tau^2$. Using the definition of $C^1$ norm, this implies that
    \[ \norm{S_i' - A_i} = O(\tau^2) \qquad \qquad  \norm{DS_i' - DA_i} = O(\tau^2), \]
    where $\ve{\cdot}$ denotes $L^\iy$ norm on $\mathcal C_i$; for matrix-valued functions $M(x)$ on $\mathcal C_i$, $\ve{M} =\sup_{x\in \mathcal C_i} \ve{M(x)}_2$, where $\ve{\cdot}_2$ denotes spectral norm.
    Again, using the definition of the $C^1$ norm,
    \begin{align*}
        & \norm{A_n \circ \cdots \circ A_0 - S'_n \circ \cdots \circ S'_0}_{C^1}\\
        \le & \norm{A_n \circ \cdots \circ A_0 - S'_n \circ \cdots \circ S'_0} + \norm{D(A_n \circ \cdots \circ A_0) - D(S'_n \circ \cdots \circ S'_0)}
    \end{align*}
    We will bound each term individually. For the first term, note that
    \begin{align*}
        & \norm{A_n \circ \cdots \circ A_0 - S'_n \circ \cdots \circ S'_0}\\
        &\le \norm{A_n \circ \cdots \circ A_1 \circ A_0 - A_n \circ \cdots \circ A_1 \circ S'_0}
        + \norm{A_n \circ \cdots \circ A_1 \circ S'_0 - S'_n \circ \cdots \circ S'_1 \circ S'_0} \tag{by triangle inequality}\\
        &=  \norm{T_{\phi- \tau, 0} \circ A_0 - T_{\phi-\tau} \circ S_0'}
        + \norm{A_n \circ \cdots \circ A_1 \circ S'_0 - S'_n \circ \cdots \circ S'_1 \circ S'_0}\\
        &\le  \norm{D T_{\timethree - \tau, 0}} \norm{S'_0 - A_0}
        + \norm{A_n \circ \cdots \circ A_1 \circ S'_0 - S'_n \circ \cdots \circ S'_1 \circ S'_0}\\
        &\le  O(\tau^2) + \norm{A_n \circ \cdots \circ A_1 \circ S'_0 - S'_n \circ \cdots \circ S'_1 \circ S'_0} \labeleqn{eqn:l5-c0}
    \end{align*}
    Observe that
    \begin{align*}
        & \sup_x \norm{A_n \circ \cdots \circ A_1 \circ S'_0 (x) - S'_n \circ \cdots \circ S'_1 \circ S'_0(x)}\\
        & = \sup_{y = S'_0(x)} \norm{A_n \circ \cdots \circ A_1 (y) - S'_n \circ \cdots \circ S'_1 (y)}\\ 
         & \le\sup_y \norm{A_n \circ \cdots \circ A_1 (y) - S'_n \circ \cdots \circ S'_1 (y)}\\
         & =\norm{A_n \circ \cdots \circ A_1 (y) - S'_n \circ \cdots \circ S'_1 (y)} \labeleqn{eqn:l5-c0-rec}
    \end{align*}
    Using \eqref{eqn:l5-c0-rec}, \eqref{eqn:l5-c0}, and induction, we get that 
    \begin{equation*}
        \norm{A_n \circ \cdots \circ A_0 - S'_n \circ \cdots \circ S'_0} \le O(n\tau^2)
    \end{equation*}
    
    Now, we bound the derivatives:
    \begin{align*}
        & \norm{D(A_n \circ \cdots \circ A_0) - D(S'_n \circ \cdots \circ S'_0)}\\
        & \le \norm{D(A_n \circ \cdots \circ A_1 \circ A_0) - D(A_n \circ \cdots \circ A_1 \circ S'_0)}\\
        & \quad + \norm{D(A_n \circ \cdots \circ A_1 \circ S'_0) - D(S'_n \circ \cdots \circ S'_1 \circ S'_0)} \tag{by triangle inequality}\\
        & = \sup_x \norm{DT_{\phi - \tau,0}\vert_{A_0(x)} DA_0(x) - DT_{\phi - \tau,0}\vert_{S_0'(x)} DS_0'(x)}\\
        & \quad + \sup_x \norm{D(A_n \circ \cdots \circ A_1) \vert_{S'_0(x)} DS'_0(x) - D(S'_n \circ \cdots \circ S'_1) \vert_{S'_0(x)} DS'_0)(x)} \tag{by chain rule}\\
        & \le \sup_x \norm{DT_{\phi - \tau,0}\vert_{A_0(x)} DA_0(x) - DT_{\phi - \tau,0}\vert_{S_0'(x)} DA_0(x)}\\
        & \quad + \sup_x \norm{DT_{\phi - \tau,0}\vert_{S_0'(x)} DA_0(x) - DT_{\phi - \tau,0}\vert_{S_0'(x)} DS_0'(x)} \tag{by triangle inequality}\\
        & \quad + \norm{DS'_0} \norm{D(A_n \circ \cdots \circ A_1)- D(S'_n \circ \cdots \circ S'_1)} \labeleqn{eqn:l5-c1-rec}\\
        & \le \sup_x \norm{DT_{\phi - \tau,0}\vert_{A_0(x)}- DT_{\phi - \tau,0}\vert_{S_0'(x)}} \norm{DA_0} \\
        & \quad + \sup_x \norm{DT_{\phi - \tau,0}\vert_{S_0'(x)}} \norm{DA_0 - DS_0}\\
        & \quad + \norm{DS'_0} \norm{D(A_n \circ \cdots \circ A_1)- D(S'_n \circ \cdots \circ S'_1)}\\
        & \le \norm{D^2 T_{\timethree - \tau, 0}} \norm{S_0' - A_0'} \norm{DA_0}
        + \norm{D T_{\timethree - \tau,0}} \norm{DA_0 - DS'_0}\\
        & \quad + \norm{DS'_0} \norm{D(A_n \circ \cdots \circ A_1)- D(S'_n \circ \cdots \circ S'_1)}\\
        & \le O(\tau^2) + \brac[\Big]{\norm{DA_0} + O(\tau^2)}\norm{D(A_n \circ \cdots \circ A_1)- D(S'_n \circ \cdots \circ S'_1)} \labeleqn{eqn:l5-c1-ds}
     \end{align*}
     where, for a 3-tensor $\mathcal{T}$, we define $\norm{\mathcal{T}} = \sup_{\norm{u} = 1} \norm{\mathcal{T} u}_2$, where $\ve{\mathcal{T} u}_2$ 
     is the spectral norm of the matrix $\mathcal{T} u$, and we define $\|D^2T_{\timethree-\tau,0}\| = \sup_{x} \norm{D^2T_{\timethree-\tau,0}(x)}$.
     In the last step, we use the fact that $\norm{D T_{s,t}}, \|D^2T_{s,t}\|$ are bounded for all $s,t > 0$; this follows from Lemma~\ref{l:DT-bound} below. (Alternatively, note that $\norm{D T_{s,t}}$ can also be more directly bounded by \Cref{t:uld-cond}.)
     
    In the above, \eqref{eqn:l5-c1-rec} follows using an argument similar to \eqref{eqn:l5-c0-rec}, \eqref{eqn:l5-c1-ds} follows since $\norm{DA_0 - DS'_0} = O(\tau^2)$. Further, differentiating \eqref{e:actual-flow-exp}, we get 
    \[ DA_0 = \id + \tau D_{(x,v)}F(x,v,t) + O(\tau^2) \]
    where $F$ denotes the defining equation of the ODE system in \eqref{eqn:ode_approx}. Therefore, we get
    \[ \norm{DA_0} \le 1 + \tau L + O(\tau^2) \] 
    where $L$ is the upper bound on $\norm{Df}$ over all the appropriate compact sets. Using this bound and induction, we get that
    \[ \norm{D(A_n \circ \cdots \circ A_0) - D(S'_n \circ \cdots \circ S'_0)} \le O(n \tau^2) (1 + \tau L +O(\tau^2))^n = O(n \tau^2 e^{n\tau L}) \]
    for small enough $\tau$. Substituting $n\tau = \timethree$, we get the overall $C^1$ bound of
    \[ \norm{A_n \circ \cdots \circ A_0 - S'_n \circ \cdots \circ S'_0}_{C^1} = O(\timethree \tau e^{\timethree L}). \]
    
    Now, we can choose $\tau$ small enough so that the two maps are $\epsilon_1$-close, finishing the proof.

     Concretely, we can write each $S'_i$ as a composition of affine-coupling maps (which constitute the $f_1, \dots, f_N$ in the lemma statement). In this manner, we can compose these compositions of affine coupling maps over each $\tau$-sized chunk of time so as to get a map which is overall close to the required flow map.
\end{proof}

\begin{lemma}\label{l:DT-bound}
Consider the ODE $\ddd t x(t)=F(x(t),t)$ for $F(x,t)$ that is $C^\ell$ in $x\in \R^d$ and continuous in $t$. Let $\mathcal C$ be a compact set and suppose solutions exist for any $(x(0),v(0))\in \mathcal C$ up to time $T$. Let $T_{s,t}$ be the flow map from time $s$ to time $t$, for any $0\le s,t\le T$. Then for any $0\le r\le \ell$, $D^r T_{s,t}$ is bounded on $T_s(\mathcal C)$.
\end{lemma}
\begin{proof}
Let $\pl_{i_1\cdots i_r} = \fc{\pl^r}{\pl x_{i_1}\cdots \pl x_{i_r}}$. 
Using the chain rule as in Lemma~\ref{l:ode_derivative}, we find by induction that
\begin{align}\label{l:Dr}
    \ddd t \pl_{i_1\cdots i_r} (T_t(x)) &=\sum_{i=1}^d \pl_i F(x(t),t) \pl_{i_1\cdots i_r} (T_t(x)_i) + G(DF,\ldots, D^{r}F, DT_t,\ldots, D^{r-1}T_t). 
\end{align}
for some polynomial $G$. 
For $r=1$, the differential equation is given by~\Cref{l:ode_derivative}. By a Gr\"onwall argument, a bound on $DF$ gives an upper and lower bound on the singular values of $DT_t$ as in~\eqref{e:dDD}.
We use induction on $r$; for $r>1$, let $v(t)$ be equal to $(\pl_{i_1\cdots i_r} (T_t(x)))_{i_1\cdots i_r}$ written as one large vector. By the chain rule and~\eqref{l:Dr},
\begin{align*}
    \ddd t \ve{v(t)}^2 &\le
    \an{|v(t)|, A|v(t)|+b}
    \le \pa{\si_{\max}(A) + \rc 2 }\ve{v(t)}^2 + \rc 2\ve{b}^2
\end{align*}
for some $A,b$ depending on $DF, \ldots, D^r F, DT_t,\ldots, D^{r-1}T_t$, where $\si_{\max}$ denotes the maximum singular value and $|v|$ denotes entrywise absolute value. Gr\"onwall's inequality (\Cref{l:gronwall}) applied to $\ve{v(t)}^2$ then gives bounds on $\ve{v(t)}^2$ and hence $\ab{\ddd t \pl_{i_1\cdots i_r} (T_t(x))}$. 
This shows $D^rT_{s,t}$ is bounded when $s\le t$ (by starting the flow at time $s$).

When $s>t$, note that the computation of the $r$th derivative of an inverse map involves up-to-$r$ derivatives of the forward map, and inverses of the first derivative. As we have a lower bound on the singular value of $DF$, this implies that $D^rT_{s,t}$ is bounded.
\end{proof}

\section{Technical Machinery}

\subsection{Proof of \Cref{lemma:discretization_error}}
\label{s:discretization_error_proof}
We consider a more general ODE than the specific one in \eqref{eqn:affine_coupling_ode}, of the form
\begin{equation}
\label{eqn:affine_euler}
    \begin{cases}
    \frac{d}{dt}(x(t)) = f(x(t), v(t) ,t) \\
    \frac{d}{dt}(v(t)) = g(x(t), v(t), t)
    \end{cases}
\end{equation}
where $f,g$ are 
$C^2$ functions in $x,v,t$. Given a compact set $\mathcal C$, suppose that the solutions are well-defined for any $(x(0),v(0))\in \mathcal C$ up to time $T$.  
Consider discretizing these ODEs into steps of size $\eta$, as follows:
\begin{equation}
\label{eqn:discretized_affine_euler}
\begin{cases}
    \widetilde T_i^x(X_i) = X_{i+1} = X_i + \eta f(X_i, 
    V_{i+1}
    , t_i) \\
    \widetilde T_i^v(V_i) = V_{i+1} = V_i + \eta g(X_i, V_i, t_i)
\end{cases}
\end{equation}
where $t_i = i \eta$. We call this the alternating Euler update. The actual flow maps are given by 
\begin{equation}\label{e:actual-flow-exp}
\begin{cases}
    T_i^x(x_i) = x_{i+1} = x_i + \eta f(x_i, v_i, t_i) +
    \int_{i\eta}^{(i+1)\eta} \int_{i\eta}^t x''(s)\,ds\,dt\\
    T_i^v(v_i) = v_{i+1} = v_i + \eta g(x_i, v_i, t_i) +
    \int_{i\eta}^{(i+1)\eta} \int_{i\eta}^t v''(s)\,ds\,dt
\end{cases}
\end{equation}
We bound the local truncation error. 
This consists of two parts. First, we have the integral terms in~\eqref{e:actual-flow-exp}:
\begin{align}\label{e:lte1}
    \ve{\coltwo{\int_{i\eta}^{(i+1)\eta} \int_{i\eta}^t x''(s)\,ds\,dt}{\int_{i\eta}^{(i+1)\eta} \int_{i\eta}^t v''(s)\,ds\,dt}}
    &\le 
    \rc 2\eta^2   \max_{s\in [0,t_i]} \ve{\coltwo{x''(s)}{v''(s)}}
    .
\end{align}
Second we bound the error from using $\widetilde v_{i+1}:=v_i + \eta g(x_i,v_i,t_i)$ instead of $v_i$ in the $x$ update,
\begin{align}\nonumber
    \ve{\eta[f(x_i, v_i + \eta g(x_i,v_i,t_i), t_i) - f(x_i,v_i,t_i)]}
    &\le 
    \ve{\eta\int_0^\eta 
    D_vf
    (x_i,v_i+sg(x_i,v_i,t_i),t_i) g(x_i,v_i,t_i)\,ds} \\
    &\le \eta^2 \max_{\mathcal C'} 
    \ve{D_vf} 
    \max_{\mathcal C'} \ve{g} .
    \label{e:lte2}
\end{align}
where 
$D_v f(x,v,t)$ denotes the Jacobian in the $v$ variables (rather than the directional derivative), and where we define
\begin{align*}
\mathcal C' 
:&=
\{(x,v+sg(x,v,t),t): (x,v) = T_t(x_0,v_0) \text{ for some }(x_0,v_0)\in \mathcal C, 0\le s\le T\},
\end{align*}
which ensures that it contains $(x_i,v_i+sg(x_i,v_i,t_i),t_i)$ and $(x_i,v_i,t_i)$. The local truncation error is then at most the sum of~\eqref{e:lte1} and \eqref{e:lte2}.

Supposing that $\coltwo fg$ is $L$-Lipschitz in $(x,v)\in \R^{2d}$ for each $t$, we obtain by a standard argument (similar to the proof for the usual Euler's method, see e.g.,~\cite[\S 16.2]{ascher2011first}) that the global error at any step is bounded by 
\begin{align}\label{e:alt-euler-bound}
    \ve{\coltwo{\widetilde x_i}{\widetilde v_i} - 
    \coltwo{x_i}{v_i}
    }
    &\le \eta \cdot \fc{e^{Lt_i}-1}L 
    \pa{
    \max_{\mathcal C'} \ve{D_v f}
    \max_{\mathcal C'} \ve{g} + \rc 2  \max_{s\in [0,t_i]} \ve{\coltwo{x''(s)}{v''(s)}}
    }.
\end{align}
In the case when $\coltwo fg$ is not globally Lipschitz, we show that we can restrict the argument to a compact set on which it is Lipschitz.
Let $\mathcal C''$ be a compact set which contains $\{(x,v,t):(x,v)=T_t(x_0,v_0)\text{ for some }(x_0,v_0)\in \mathcal C, 0\le s\le T\}$ in its interior. Apply the argument to $\hat{f}$ and $\hat g$ which are defined to be equal to $f, g$ on $\mathcal C''$, and are globally Lipschitz. Then the error bound applies to the system defined by $\hat f, \hat g$. Hence, for small enough step size, the trajectory of the discretization stays inside $\mathcal C''$, and is the same as that for the system defined by $f,g$. Then~\eqref{e:alt-euler-bound} holds for small enough $\eta$ and $L$ equal to the Lipschitz constant in $(x,v)$ on $\mathcal C''$.



To get a bound in $C^1$ topology, we need to bound the derivatives of these maps as well. Let $T_{s,t}(x,v)$ 
denote the flow map of system \eqref{eqn:affine_euler}. Let $h(x,v,t) = (f(x,v,t),g(x,v,t))$. Now, consider the system of ODEs
\begin{equation}
    \label{eqn:combined_odes}
    \begin{cases}
        \frac{d}{dt}(x(t))  = f(x(t),v(t),t)\\
        \frac{d}{dt}(v(t))  = g(x(t),v(t),t)\\
        \frac{d}{dt}(\alpha(t))  = D_{(x,v)}f (x(t),v(t),t) \coltwo{\alpha(t)}{\beta(t)}\\
        \frac{d}{dt}(\beta(t))  = D_{(x,v)}g (x(t),v(t),t)\coltwo{\alpha(t)}{\beta(t)}
    \end{cases}
\end{equation}
where $\alpha(t), \beta(t)$ are $d \times 2d$ matrices. Note that setting $\coltwo{\alpha(0)}{\beta(0)} = \id_{2d}$ and $\coltwo{\alpha(t)}{\beta(t)} = D_{(x,v)} T_{0,t} (x(0),v(0))$ satisfies \eqref{eqn:combined_odes} by \Cref{l:ode_derivative}.

Now we claim that applying the alternating Euler update to $(x,\alpha), (v,\beta)$, the resulting $(\alpha_i,\beta_i)$ is exactly the Jacobian of the flow map that arises from alternating Euler applied to $x,v$. This means that we can bound the errors for $\alpha,\beta$ using the bound for the alternating Euler method.

The claim follows from noting that the alternating Euler update on $\alpha$, $\beta$ is
\begin{align*}
    \alpha_{i+1} &= 
    (\id_d,O) 
    + 
    D_{(x,v)}f(x_i,v_{i+1},t_i) \coltwo{\al_i}{\be_{i+1}}\\
    \beta_{i+1} &= 
    (O,\id_d) 
    + 
    D_{(x,v)}f(x_i,v_{i},t_i) \coltwo{\al_i}{\be_i},
\end{align*}
which is the same recurrence that is obtained from differentiating $X_{i+1},V_{i+1}$ in~\eqref{eqn:discretized_affine_euler} with respect to $X_0,V_0$, and using the chain rule.

Thus we can apply~\eqref{e:alt-euler-bound} to get a bound for the Jacobians of the flow map. The constants in the $O(\eta)$ bound depend on up to the second derivatives of the $x,v,\alpha,\beta$ for the true solution, Lipschitz constants for $\coltwo fg$, $D\coltwo fg$ (on a suitable compact set), and bounds for $D_vf, g, D_vD_{(x,v)}f, D_{(x,v)}g$  (on a suitable compact set).

\subsection{Wasserstein bounds}
\label{s:wasserstein_bounds}

\begin{lemma}
\label{l:wasserstein_bound_1}
Given two distributions $p,q$ and a function $g$ with Lipschitz constant $L = \Lip(g)$,
\[ W_1(g_\# p, g_\# q) \le L W_1(p,q) \]
\end{lemma}
\begin{proof}
    Let $\epsilon >0$. Then there exists a coupling $(x,t) \sim \gamma$ such that 
    \[ \int \norm{x-y}_2 d\gamma(x,y) \le W_1(p,q) + \epsilon \]
    Consider the coupling $(x',y')$ given by $(x',y') = (g(x),g(y))$ where $(x,y) \sim \gamma$. Then
    \begin{align*}
        W_1(g_\# p, g_\#q)
        & \le \int \norm{g(x) - g(y)}_2 \, d\gamma(x,y) \\
        & \le \Lip(g) \int \norm{x - y} \, d\gamma(x,y)\\
        & \le L W_1(p,q) + L \epsilon.
    \end{align*}
    Since this holds for all $\epsilon > 0$, we get that
    \[ W_1(g_\# p, g_\#q) \le LW_1(p,q) \]
\end{proof}

\begin{lemma}
\label{l:wasserstein_bound_2}
Given two functions $f,g:\R^d \to \R^d$ that are uniformly $\epsilon_1$-close over a compact set $\mathcal{C}$ in $C^1$ topology, and a probability distribution $p$,
\[ W_1(f_\# (p|_\mathcal{C}), g_\# (p|_\mathcal{C})) \le \epsilon_1 \]
\end{lemma}
\begin{proof}
    Consider the coupling $\gamma$, where a sample $(x,y) \sim \gamma$ is generated as follows: first, we sample $z \sim p|_\mathcal{C}$, and then compute $x=f(z)$, $y=g(z)$. By definition of the pushforward, the marginals of $x$ and $y$ are $f_\#(p|_\mathcal{C})$ and $g_\#(p|_\mathcal{C})$ respectively. However, we are given that for this $\gamma$, $\|x-y\| \le \epsilon_1$ uniformly. Thus, we can conclude that
\begin{align*}
    W_1(f_\#(p|_\mathcal{C}), {g}_\#(p|_\mathcal{C})) &\le \int_{\R^d \times \R^d}\|x-y\|_2 \,d\gamma(x,y) \\
    &\le \int_{\R^d \times \R^d}\epsilon_1\, d\gamma(x,y)
    =\epsilon_1
\end{align*}
\end{proof}

\subsection{Proof of \Cref{l:conditional_wasserstein}}
\label{s:conditional_wasserstein_proof}
\begin{proof}
    Fix any $R > 0$, and set $\mathcal{C} = B(0, R)$. Consider the coupling $(X,Y) \sim \gamma$, where a sample $(X,Y)$ is generated as follows: we first sample $X \sim p^* = \mathcal{N}(0, \id_{2d})$. If $X \in B(0, R)$, then we set $Y=X$. Else, we draw $Y$ from $p^*|_\mathcal{C}$. Clearly, the marginal of $\gamma$ on $X$ is $p$. Furthermore, since $p^*$ and $p^*|_\mathcal{C}$ are proportional within $\mathcal{C}$, the marginal of $\gamma$ on $Y$ is $p^*|_\mathcal{C}$. Then, we have that
    \begin{align*}
        W_1(p^*, p^*|_\mathcal{C}) &\le \int_{\R^{2d} \times \mathcal C}  \|x-y\|d\gamma \\
        &
        = \cancel{\int_{\mathcal{C} \times \mathcal{C}} \|x-y\|d\gamma} + 
        \int_{\mathbb{R}^{2d}\setminus\mathcal{C} \times \mathcal{C}}\|x-y\|d\gamma 
        \\
        &
        = \int_{\mathbb{R}^{2d}\setminus\mathcal{C} \times \mathcal{C}}\|x-y\|d\gamma \\
        &
        \le \int_{\mathbb{R}^{2d}\setminus\mathcal{C} \times \mathcal{C}}(\|x\| + \|y\|)d\gamma \\
        &
        \le \int_{\mathbb{R}^{2d}\setminus\mathcal{C} \times \mathcal{C}}(\|x\| + R)d\gamma \\
        &
        \le \int_{\mathbb{R}^{2d}\setminus\mathcal{C} \times \mathcal{C}}(\|x\| + R)d\gamma \\
        &
        = \int_{\R^{2d} \setminus \mathcal{C}} (\|x\| + R)dp^* \\
        &
        \le \int_{\R^{2d} \setminus \mathcal{C}} 2\|x\|dp^* = \frac{2}{\sqrt{2\pi}}\int_{\R^{2d} \setminus \mathcal{C}} \|x\|e^{-\frac{\|x\|^2}{2}}dx
    \end{align*}
    Now, note that $\int_{\R^{2d} } \|x\|e^{-\frac{\|x\|^2}{2}}dx < \infty$.
    Hence, by the Dominated Convergence Theorem, \[\lim_{R\to \iy}\int_{\R^{2d} \setminus B(0,R)} \|x\|e^{-\frac{\|x\|^2}{2}}dx = 0.\]
    Thus, given any $\delta >0$, we can choose $R$ large enough so that the integral above is smaller than $\delta$, which concludes the proof.
\end{proof}

\subsection{Derivatives of flow maps}
\label{s:derivative_flow_map}
We state and prove a technical lemma about the ODE that the derivative of a flow map satisfies.
\begin{lemma}
    \label{l:ode_derivative}
    Suppose $x_t = x(t)$ satisfies the ODE
    \[ \dot{x} = F(x,t) \]
    with flow map $T(x,t) : \R^n \times \R \to \R^n$. Suppose $\alpha(t)$ be the derivative of the map $x \mapsto T(x,t)$ at $x_0$, then $\alpha(t)$ satisfies
    \[ \dot{\alpha} = DF(x_t,t) \alpha \]
    with $\alpha(0) = \id$.
\end{lemma}
\begin{proof}
    Let $T_t(x) = T(x,t)$. Then $T_t$ satisfies
    \[ T_t(x_0) = \int_0^t F(x_s,s)\, ds. \]
	Differentiating, we get
	\begin{align*}
		\alpha(t) = DT_t(x_0)
		& = \int_0^t D(F(x_s,s)) \,ds \\
		& = \int_0^t DF(x_s,s) DT_s(x_0) \,ds \tag*{by chain rule}\\
		& = \int_0^t DF(x_s,s) \alpha(s)\, ds.
	\end{align*}
	Now, looking at the derivative with respect to $t$, we get
	\[ \dot{\alpha} = DF(x_t,t) \alpha, \]
	which is the required result.
\end{proof}

\subsection{Solving Perturbed ODEs}
\label{s:perturbation_analysis}
In this section, we state a result about finding approximate solutions of perturbed differential equations. Consider the ODE having the following general form:
\[ \dot{x} = Ax + \epsilon g(x,t) \]
The reason we are concerned with this ODE is that the ODE given by \Cref{eqn:affine_coupling_ode} has precisely this form, namely with $x \equiv \coltwo xv$, $A \equiv \begin{bmatrix} 0 & \id_d \\ -\text{diag}(\Omega^2) & 0 \end{bmatrix}$ and $\epsilon g(x,t)\equiv -\tau\begin{bmatrix} F(v,t) \odot x \\ J(x,t) + G(x,t) \odot v\end{bmatrix}$. 

Let $T^x: \R \times \R^n \to \R^n$ be the time $t$ flow map for this ODE. We will find a flow map $T^y: \R \times \R^n \to \R^n$ such that the maps $T_t^x$ defined by $T_t^x(x) = T^x(t,x)$ and the map $T_t^y$ defined by $T_t^y(y) = T^y(t,y)$ are uniformly $\epsilon$-close over $\mathcal{C}$ in $C^r$ topology for all $0 \le t \le 2 \pi$. That is, 
\[ \sup_{x} \qquad \norm{T_t^x(x) - T_t^y(x)} + \norm{DT_t^x(x) - DT_t^y(x)} + \dots + \norm{D^rT_t^x(x) - D^rT_t^y(x)}\]
is small, for all $t \in [0,2\pi]$. Here $D^r$ denotes the $r$-th derivative, and the norms are defined inductively as follows: for a $r$-tensor $\mathcal T$, we let $\ve{\mathcal T}=\sup_{\ve{u}=1}\ve{\mathcal Tu}$; here $\mathcal Tu$ is a $(r-1)$-tensor. (The choice of norm is not important; we choose this for convenience.)

\begin{lemma}
\label{l:perturb_trajectory_bound}
    Consider the ODE
    \begin{align}\label{e:ode-perturb-exist}
        \ddd t x(t) &= F(x(t),t) + \ep G(x(t),t)
    \end{align}
    where $x:[0,t_{\max}]\to \R^n$, $F,G:\R^n\times \R\to \R^n$, and 
    $F(x,t)$, $G(x,t)$ are $C^1$, and $F$ is $L$-Lipschitz. Let $\mathcal C$ be a compact set, and suppose that for all $x_0\in \mathcal C$, solutions to~\eqref{e:ode-perturb-exist} with $x(0)=x_0$ exist for $0\le t\le t_{\max}$ and $\ep=0$.
    Then there exists $\ep_0$ such that solutions to~\eqref{e:ode-perturb-exist} with $x(0)=x_0$ exist for $0\le t\le t_{\max}$ and $0\le \ep<\ep_0$.
    
    Moreover, letting $x^{(\ep)}(t)$ be the solution with given $\ep$, we have that as $\ep\to 0$, $\ve{x^{(\ep)}(t)-x^{(0)}(t)}=O(\ep)$, where the constants in the $O(\cdot)$ depend only on 
    $L$ and $\max_{0\le t\le t_{\max}, x_0\in \mathcal C} \ve{G(x^{(0)}(t),t)}$ (the maximum of $G$ on the $\ep=0$ trajectories).
\end{lemma}
\begin{proof}  
    Let $T^\epsilon(t,x_0)$ be the flow map of \eqref{e:ode-perturb-exist}. Let $\mathcal{K} = T^0(\mathcal{C} \times [0,t_{\max}])$ be the image of $\mathcal{C} \times [0,t_{\max}]$ under the flow map $T^0$. Since $F$ is $C^1$, $T^0$ is $C^1$, which implies that $\mathcal{K}$ is bounded. Fix some $\epsilon_2 > 0$. Let $B(\mathcal{K},r)$ denote the set
    \[ B(\mathcal{K},r)  = \set{(x,t) \in \R^n \times [0,t_{\max}] : d(\mathcal{K}, x) \le r} \]
    Let $\mathcal{K}_2 = B(\mathcal{K},\epsilon_2)$. Note that since $\mathcal{K}$ is compact, so is $\mathcal{K}_2$. Let 
    \[ M = \max \set*{\sup_{(x,t) \in \mathcal{K}_2 \times [0,t_{\max}]} \norm{F(x,t)}, \sup_{(x,t) \in \mathcal{K}_2 \times [0,t_{\max}]} \norm{G(x,t)}} \]
    $M$ is finite since $\mathcal{K}_2$ is compact and $F,G$ are $C^1$.
    
    Let $h: \R \to \R$ be a $1$-Lipschitz $C^1$ function such that
    \begin{align*}
    h(x) &= 
    x  \text{ if }|x| \le M\\
    |h(x)| &\le 2M \text{ for all }x.
    \end{align*}
    Let $h_n:\R^n \to \R^n$ be defined as $h_n(x) = \frac{x}{\norm{x}} h(\norm{x})$. Then $h_n(x)$ is also $C^1$ and is the identity function on $B(0,M)$. Let $F_1 = h_n \circ F$ and let $G_1 = h_n \circ F$. Then $F_1, G_1$ are $C^1$ functions such that $\norm{F_1}, \norm{G_1} \le 2M$. Further, $F_1$ is $L$-Lipschitz. Now, we look at the ODE
    \begin{align}\label{e:ode-perturb-clamped}
        \ddd t x(t) &= F_1(x(t),t) + \ep G_1(x(t),t)
    \end{align}

    Since $F_1,G_1$ are $C^1$, note that the function $H_1(x,\epsilon,t) = F_1(x,t) + \epsilon G_1(x,t)$ is $C^1$ in $x,t,\epsilon$. Therefore, using the existence theorem for parametric ODEs (Theorem 1.2, \cite{chicone2006ordinary}), there is a $\epsilon_1, t_1 > 0$ such that solutions $x_1^{(\epsilon)}(t)$ to \eqref{e:ode-perturb-clamped} exist for all $x_0 \in \mathcal{C}, \epsilon < \epsilon_1$ and $t < t_1$. Further, the extensibility result for the ODEs (Theorem 1.4, \cite{chicone2006ordinary}) states that if $t_1$ is largest such value for which such solutions exist, then there exists a $x_0 \in \mathcal{C}$ and $\epsilon < \epsilon_1$ such that $\lim_{t \to t_1} \norm{x_1^{(\epsilon)}(t)} = \infty$. 
    
    Now, we will bound $\norm{x_1^{(\epsilon)} - x_1^{(0)} }$ for $t < t_1$. Define $\alpha = x_1^{(0)} - x_1^{(\epsilon)}$. Then $\alpha(t)$ satisfies
    \[ \ddd{t} \alpha(t) = F_1(x_1^{(0)}(t), t) - F_1(x_1^{(\epsilon)}(t),t) - \epsilon G_1(x_1^{(\epsilon)}(t),t) \]
    
    Therefore,
    \begin{align*}
        \ddd{t} \norm{\alpha(t)}^2 
        & \le 2 \norm{\alpha(t)} \norm*{\ddd{t} \alpha(t)}\\
        & \le 2 \norm{\alpha(t)} \norm{F_1(x_1^{(0)}(t), t) - F_1(x_1^{(\epsilon)}(t),t) - \epsilon G_1(x_1^{(\epsilon)}(t),t)}\\
        & \le 2 \norm{\alpha(t)} \brac{L \norm{\alpha(t)} + 2 \epsilon M} \\
        & \le 2 L \norm{\alpha(t)}^2 + 4\epsilon M \norm{\alpha(t)}\\
        \implies \ddd{t} \norm{\alpha(t)} & \le
        \rc 2 \ve{\al(t)}^{-1} \ddd{t} \norm{\alpha(t)}^2  \le 
        L \norm{\alpha(t)} + 2 \epsilon M
     \end{align*}

    Now, Gr\"onwall's inequality (\Cref{l:gronwall}) gives us the bound 
    \begin{equation}
        \label{e:ode-perturb-bound}
        \norm{\alpha(t)} \le 2 \epsilon t M e^{Lt} \le 2 \epsilon t_{\max} M e^{Lt_{\max}} = O(\epsilon)
    \end{equation}
    
    Since $t_{\max}, L, M$ are fixed, we can choose $\epsilon_0$ such that $\epsilon_0 < \epsilon_1$ and $2 \epsilon_0 t_{\max} M e^{L t_{\max}} < \epsilon_2$, which ensure that for all $x_0 \in \mathcal{C}, \epsilon < \epsilon_0$ and $t < \min (t_1, t_{\max})$, the point $x_1^{(\epsilon)}(t)$ is in the interior of $\mathcal{K}_2$. Therefore, if $t_1 \le t_{\max}$ then $\lim_{t \to t_1} \norm{x_1^{(\epsilon)}(t)} \in \mathcal{K}_2$, which contradicts the extensibility result. Thus, $t_1 > t_{\max}$, and hence flow maps for \eqref{e:ode-perturb-clamped} exists for all $0 \le \epsilon \le \epsilon_0$ and $0 \le t \le t_{\max}$.
    
    Now, we end with the remark that since $F_1 = F$ and $G_1 = G$ in $\mathcal{K}_2$, the flow map of \eqref{e:ode-perturb-clamped} is a flow map for \eqref{e:ode-perturb-exist} inside $\mathcal{K}_2$, and therefore, solutions to \eqref{e:ode-perturb-exist} exist for all $x_0 \in \mathcal{C}, 0 \le \epsilon \le \epsilon_0$ and $0 \le t \le t_{\max}$.
    
    Lastly, we will comment on value of $M$. Let $G$ be $L_1$-Lipschitz on $\mathcal{K}_2$, and let 
    \[ M' = \max_{0 \le t \le t_{\max}, x_0 \in \mathcal{C}} \norm{G(x^{(0)}(t), t)}\]
    Then $M \le M' + \epsilon_0 L_1$. Therefore, we can just choose $\epsilon_0$ small enough so that $M \le 2M' + 1$, which enforces the constants in $O(\cdot)$ notation to depend only on $L, M'$ and $t_{\max}$.

    
\end{proof}

\ifdefined\PROOF
\begin{lemma}
	\label{l:linear_ode_upper_bound}
	Any solution $w(t)$ of $\dot{w} \le aw + b$ for constants $a,b$ satisfies
	\[ w(t) \le \frac{b}{a} \brac{e^{at} - 1} + w(0)e^{at}. \]
\end{lemma}
\begin{proof}
	The solution of the corresponding differential equation is given by using an integrating factor.
	\begin{align*}
		\dot{w} - aw & \le b \\
		\implies \dot{w}e^{-at} - awe^{-at} & \le be^{-at} \\
		\implies \frac{d}{dt} (we^{-at}) & \le be^{-at} \\
		\implies w(t)e^{-at} - w(0) & \le \frac{b}{a} \brac{1 - e^{-at}} \\
		\implies w(t) & \le \frac{b}{a} \brac{e^{at} - 1} + w(0)e^{at} 
	\end{align*}
	This proves the bound that we want.
\end{proof}
\fi

\begin{lemma}
\label{l:ode-perturb-cr}
   	Consider the ODE's
	\begin{align}
	\label{e:ode-perturb-orig}
	\ddd t x(t) &= F(x(t),t) + \epsilon G(x(t),t)\\
	\nonumber
	\ddd t y_0(t) &= F(y_0(t),t)\\
	\nonumber 
	\ddd t y(t) &= F(y(t),t) + \ep G(y_0(t),t)
	\end{align}
	such $F,G:\R^n\times \R\to \R^n$ are in $C^{r+1}$. 
	Let $\mathcal C\subseteq \R^n$ be a compact set, and suppose that solutions to~\eqref{e:ode-perturb-orig} exist for all $x_0\in \mathcal C$. 
	Let $T^x(x_0)$, $T^{y_0}(x_0)$, and $T^{y}(x_0)$ be the time $t_{\max}$-flow map corresponding to this ODE for initial values $x(t)=y_0(t)=y(t)=x_0$.
	
	Then as $\ep\to 0$, the maps $T^x_t$ and $T^y_t$ are
	$O(\epsilon^2)$ uniformly close over $\mathcal{C}$ in $C^r$ topology, for all $t \in [0,t_{\max}]$. The constants in the $O(\cdot)$ depend on 
	$\max_{0\le k\le r+1, x_0\in \mathcal C, 0\le t\le t_{\max}} \ve{D^k F(x,t)|_{x=y_0(t)}}$ (the first $r+1$ derivatives of $F$ on the $y_0$-trajectories)
	and
	$\max_{0\le k\le r, x_0\in \mathcal C, 0\le t\le t_{\max}} \ve{D^kG(x,t)|_{x=y_0(t)}} $, 
	(the first $r$ derivatives of $G$ on the $y_0$-trajectories).
\end{lemma}
\begin{proof}
    Let $F_{\epsilon}(x,t) = F(x,t) + \epsilon G(x,t)$, and let $T_t^\epsilon(x_0)$ denote the flow map of~\eqref{e:ode-perturb-orig} starting at $x_0$.
    From~\eqref{l:Dr}, there is a polynomial $P = P_{i_1, \ldots, i_r}$ such that 
    \begin{equation}
        \ddd{t} \partial_{i_1 \cdots i_r} T_t^{x}(x_0) = \sum_{i=1}^d \partial_i F_\epsilon(x(t),t) \partial_{i_1 \cdots i_r} T_{t,i}^\epsilon + P(DF_\epsilon, \ldots, D^r F_\epsilon, DT_t^x, \ldots, D^{r-1} T_t^x)    
    \end{equation}
    On the other hand, applying~\eqref{l:Dr} to $y_0$ gives
    \begin{equation*}
        \ddd{t} \partial_{i_1 \cdots i_r} T_t^{y_0}(x_0) = \sum_{i=1}^d \partial_i F(y_0(t),t) \partial_{i_1 \cdots i_r} T_{t,i}^{y_0} + P(DF, \ldots, D^r F, DT_t^{y_0}, \ldots, D^{r-1} T_t^{y_0})
    \end{equation*}
    We will now show that these two trajectories are $O(\epsilon)$ uniformly close by induction on $r$. Note that the base case ($r=0$) is proved in \Cref{l:perturb_trajectory_bound}. We will first show that 
    \[ \norm{P(DF_\epsilon, \ldots, D^r F_\epsilon, DT_t^x, \ldots, D^{r-1} T_t^x) - P(DF, \ldots, D^r F, DT_t^{y_0}, \ldots, D^{r-1} T_t^{y_0})} = O(\epsilon) \]
    Since $P$ is a fixed polynomial that depends on $i_1, \ldots, i_r$, to show the above, we only need to show that the 
    coordinates are $O(\epsilon)$ close, for small enough $\epsilon$.
    \begin{align*}
      \norm{D^k F_\epsilon(x(t),t) - D^k F(y_0(t),t)}
      & \le \norm{D^k F_\epsilon(x(t),t) - D^k F(x(t),t)} + \norm{D^k F(x(t),t) - D^k F(y_0(t),t)}\\
      & \le \epsilon \norm{D^k G(x(t),t)} +  \norm{x(t) - y_0(t)} (2N_{k+1} + 1)\\
      & \le O(\epsilon (2 M_k + 2 N_{k+1} + 2))
    \end{align*}
    where $N_{k+1} = \sup_{x_0 \in \mathcal{C}, 0 \le t \le t_{\max}} \norm{D^{k+1} F(x,t)\vert_{x = y_0(t)}}$ and 
    $M_k = \sup_{x_0 \in \mathcal{C}, 0 \le t \le t_{\max}} \norm{D^k G(x,t)\vert_{x = y_0(t)}}$. The second inequality follows since the base case (\Cref{l:perturb_trajectory_bound}) implies that $\norm{x(t) - y_0(t)} = O(\epsilon)$, and since $D^{k+1}F$ is continuous, it follows that for small enough $\epsilon$, $\norm{D^{k+1}F\vert_{(x,t)}} \le 2 N_{k+1} + 1$, for all $x$ such that $\norm{x - y_0(t)} = O(\epsilon)$. Similarly, note that for small enough $\epsilon$, $\norm{D^k G(x(t),t)} \le 2M_k + 1$, since $G$ is $C^k$.
    Therefore, $\norm{D^k F_\epsilon(x(t),t) - D^k F(y_0(t),t)} = O(\epsilon)$, where constants in $O(\cdot)$ depend $M_k$ and $N_{k+1}$.
    
    To simplify notation, let $\alpha(t) = \ddd{t} \partial_{i_1 \cdots i_r} (T_t^x - T_t^{y_0})$. Then,
    \begin{align*}
        \ddd{t} \alpha(t)
        & = \ddd{t} \partial_{i_1 \cdots i_r} (T_t^x - T_t^{y_0})\\
        & = \sum_{i=1}^d \partial_i F_\epsilon(x(t), t) \partial_{i_1 \cdots i_r} T_{t,i}^x - \sum_{i=1}^d \partial_i F(y_0(t), t) \partial_{i_1 \cdots i_r} T_{t,i}^{y_0} + O(\epsilon)\\
        & = \sum_{i=1}^d \partial_i F_\epsilon(x(t), t) \partial_{i_1 \cdots i_r} (T_{t,i}^x - T_{t,i}^{y_0}) + \sum_{i=1}^d (\partial_i F_\epsilon(x(t), t) - \partial_i F(y_0(t),t)) \partial_{i_1 \cdots i_r} T_{t,i}^{y_0} + O(\epsilon) \\
        & = DF_\epsilon(x(t),t) \partial_{i_1 \cdots i_r} (T_t^x - T_t^{y_0}) + (DF_\epsilon(x(t),t) - DF(y_0(t),t)) \partial_{i_1 \cdots i_r} T_t^x + O(\epsilon) \\
        & = DF_{\epsilon}(x(t),t) \alpha(t) + (DF(x(t),t) - DF(y_0(t),t) + \epsilon G(x(t),t)) \partial_{i_1 \cdots i_r} T^{y_0}_t + O(\epsilon)\\
        \Rightarrow \frac{1}{2} \ddd{t} \norm{\alpha}^2
        & \le \norm{DF_\epsilon(x(t),t)} \norm{\alpha}^2 + O(\epsilon(N_2 + M_0)) \norm{\partial_{i_1 \cdots i_r} T_t^{y_0}}  + O(\epsilon)\\
        \Rightarrow \ddd{t} \norm{\alpha}
        & \le \norm{DF(x(t),t)} \norm{\alpha} + O(\epsilon)\\
        & \le (2N_1 + 1) \norm{\alpha} + O(\epsilon)
    \end{align*}
    Now, Gr\"onwall's inequality (\Cref{l:gronwall}) gives us the bound,
    \[ \norm{\alpha(t)} \le t_{\max} e^{N_1 t_{\max}} O(\epsilon) = O(\epsilon)\]
    The constants in the last $O(\cdot)$ notation depend on $t_{\max}$, $N_k$ for $0 \le k \le r+1$ and $M_k$ for $0 \le k \le r$.  
    
    This tells us that 
    \begin{equation}
    \label{e:ode-perturb-eps}
        \norm{T_t^x - T_t^{y_0}}_{C^r} = O(\epsilon)
    \end{equation}
    Now, note that $T_t^y$ satisfies
    \begin{align*}
        \ddd{t} y(t)
        & = F(y(t),t) + \epsilon G(y(t),t) + \epsilon(G(y_0(t),t) - G(y(t),t))\\
        \implies \ddd{t} y(t) 
        & = F(y(t),t) + \epsilon G(y(t),t) + \epsilon^2 H(y(t),t)
    \end{align*}
    where $H(y,t) = \frac{1}{\epsilon} (G(y_0(t),t) - G(y(t),t))$.
    Consider the system of ODEs
    \begin{equation}
    \label{e:ode-perturb-gamma}
        \ddd{t} y(t) = F_\epsilon(y(t), t) + \gamma H(y(t),t)
    \end{equation}
    Note that when $\gamma = 0$, $T_t^x$ is the flow map for this system, and when $\gamma = \epsilon^2$, $T_t^y$ is the flow map for this system. Therefore, applying~\eqref{e:ode-perturb-eps} for the system~\eqref{e:ode-perturb-gamma}, we get
    \[ \norm{T_t^x - T_t^y}_{C^r} = O(\gamma) = O(\epsilon^2) \]
    where the constants in $O(\cdot)$ notation depend on
    $\sup_{0 \le k \le r, x_0 \in \mathcal{C}, 0 \le t \le t_{\max}} \norm{D^{k+1} F_\epsilon(x(t),t)}$ which is bounded by $\max_{0\le k\le r}(2N_{k+1} + 1)$ for small $\epsilon$, and $M'_k = \sup_{0 \le k \le r, x_0 \in \mathcal{C}, 0 \le t \le t_{\max}} \norm{D^{k+1} H(x(t),t)}$. Using the definition of $H$,
    \begin{align*}
        \norm{D^k H(x(t),t)}
        & = \frac{1}{\epsilon} \norm{D^k G(y_0(t),t) - D^k G(x(t),t)}\\
        & \le \frac{1}{\epsilon} \norm{y_0(t) - x(t)} (2M_{k+1} + 1)\\
        & =  \frac{1}{\epsilon} \cdot O(\ep)\cdot 
        (2M_{k+1} + 1) = O(1)
    \end{align*}
    where the constant in the $O(\cdot)$ depends on $M_0, \ldots, M_{r+1}$ and $N_1, \ldots, N_{r+1}$. This proves the dependence in $O(\cdot)$ notation as stated in the statement, completing the proof.
\end{proof}

\begin{corollary} 
	\label{l:perturbed_ode_cr}
	Consider the ODE
	\[ \dot{x} = Ax + \epsilon g(x,t) \]
	such that $\norm{A} = 1$ and $g$ has bounded $(r+1)^{th}$ derivatives on a compact set $\mathcal{C}$. Let $T^x$ be the flow map corresponding to this ODE. For fixed $x_0$, let $y_0, y_1$ be functions satisfying 
	\begin{align*} \dot{y_0} &= Ay_0 \\
	 \dot{y_1} &= Ay_1 + g(y_0(t),t) \end{align*}
	such that $y_0(0) = x_0$ and $y_1(0) = 0$. Consider the flow map $T^y : \R \times \R^n$ such that
	$T^y(t,x_0) = y_0(t) + \epsilon y_1(t)$. Then, the maps $T^x_t$ and $T^y_t$ are
	$O(\epsilon^2)$ uniformly close over $\mathcal{C}$ in $C^r$ topology, for all $t \in [0,2\pi]$. The constants in the $O(\cdot)$ depend on 
	$\ve{A}$ and the first $r$ derivatives of $g$ on 
	the trajectories $x(t) = e^{At}x_0, x_0\in \mathcal C$.
\end{corollary}
This follows directly from \Cref{l:ode-perturb-cr}, after noting $\dot y = Ay_0+\ep Ay_1 + \ep g(y_0(t),t) = Ay + \ep g(y_0(t),t)$. Note that $F(x) = Ax$ is a linear function, so derivatives of $F$ are bounded, and the $y_0$ trajectories can be computed easily.
\ifdefined\PROOF
\begin{proof}
	Since we are working in $C^2$ topology, we need to show that
	\[ \sup_{t \in [0,2\pi], x_0} \norm{T_t^x(x_0) - T_t^y(x_0)} + \norm{DT_t^x(x_0) - DT_t^y(x_0)} = O(\epsilon^2) \]
	First, we will show that
	\[ \sup_{t \in [0,2\pi], x_0} \norm{T_t^x(x_0) - T_t^y(x_0)} = O(\epsilon^2) \]
	Let $z(t) = T_t^x(x_0) - T_t^y(x_0)$, and let $w(t) = \norm{z_0}$. Then $w^2 = z^\intercal z$. Therefore,
	\begin{align*}
		\frac{d}{dt} (w^2) & = 2 z^\intercal \dot{z} \\
											 & = 2 z^\intercal \brac{Ax + \epsilon g(x,t) - Ay_0 - \epsilon Ay_1 - \epsilon g(y_0(t),t)} \\
											 & = 2 z^\intercal A z + 2 z^\intercal \epsilon \brac{g(x,t) - g(y_0,t)} \\
											 & \le 2 \norm{A} \norm{z}^2 + 2 \epsilon \norm{z} L_g \norm{x - y_0} \\
											 & = 2 \norm{A} \norm{z}^2 + 2 \epsilon \norm{z} L_g \norm{x - y + \epsilon y_1}\\
											 & \le \brac{2 \norm{A} + 2 \epsilon L_g} \norm{z}^2 + 2 \epsilon^2 L_g \norm{z} \norm{y_1}\\
											 \implies
		2 w \dot{w} & \le \brac{2 \norm{A} + 2 \epsilon L_g} w^2 + 2 \epsilon^2 L_g w \norm{y_1}\\
		\implies
		\dot{w} & \le \brac{\norm{A} + \epsilon L_g} w + 2 \epsilon^2 L_g \norm{y_1} 
	\end{align*}
	Let $\sup_{x_0,t} \norm{y_1} = M$, then using \Cref{l:linear_ode_upper_bound} with $a = \norm{A} + \epsilon L_g$ and $b = 2 \epsilon^2 L_g M$, we get an upper bound
	\[ w \le \epsilon^2 \brac{\frac{L_g M}{1 + \epsilon L_g} \brac{e^{2 \pi(1 + \epsilon L_g)} - 1}} = O(\epsilon^2) \]

	For the upper bound on $\norm{y_1}$, note that $\norm{y_1}$ satisfies
	\begin{align*}
		\frac{d}{dt} (\norm{y_1}^2)
	& = 2 y_1^\intercal A y_1 + 2 y_1^\intercal g(y_0,t) \\
	& \le 2 \norm{A} \norm{y_1}^2 + 2 \norm{y_1} M_g \\
	\implies
		\frac{d}{dt} (\norm{y_1}) & \le \norm{y_1} + M_g\\
		\implies
		\norm{y_1} & \le M_g \brac{e^{2 \pi} - 1}
	\end{align*}
	where $M_g = \sup_{x,t} = \norm{g(x,t)}$.

	\begin{equation}
		\label{eq:pertubation_bound}
		\norm{x - y} \le \epsilon^2 \brac{\frac{L_g M_g}{(1 + \epsilon L_g)} \brac{e^{2 \pi} - 1} \brac{e^{2\pi(1 + \epsilon L_g)} - 1}}
	\end{equation}

	Now, we will bound the derivatives, and show that
	\[ \sup_{t \in [0,2\pi], x_0} \norm{DT_t^x(x_0) - DT_t^y(x_0)} = O(\epsilon^2) \]
	Define $\alpha(t) = DT_t^x(x_0)$. then $\alpha(t)$ satisfies the ODE
	\begin{equation}
		\label{eq:x_derivative}
		\dot{\alpha} = A \alpha + \epsilon Dg(x(s),s) \alpha
	\end{equation}
	with the initial condition $\alpha(0) = \id$.
	This follows since 
	\[ T_t^x(x_0) = \int_0^t A T_s^x(x_0) + \epsilon g(T_s^x(x_0),s) ds \]
	We have
	\begin{align*}
		\alpha(t) = DT_t^x(x_0)
		& = \int_0^t D \brac{A T_s^x(x_0) + \epsilon g(T_s^x(x_0),s)} ds \\
		& = \int_0^t A DT_s^x(x_0) + \epsilon Dg(T_s^x(x_0),s) DT_s^x(x_0) ds \\
		& = \int_0^t A \alpha(s) + \epsilon Dg(x(s), s) \alpha(s) ds
	\end{align*}

	Let $T_t^{y_0}$ be the flow map of ODE $\dot{y_0} = A y_0$. Define $T_t^{y_1}(x_0) = y_1(t)$. Then we have
	\[ DT_t^{y}(x_0) = DT_t^{y_0}(x_0) + \epsilon DT_t^{y_1}(x_0) \]
	Now, define $\beta_0(t) = DT_t^{y_0}(x_0)$, and $\beta_1(t) = DT_t^{y_1}(x_0)$, and $\beta(t) = DT_t^{y}(x_0) = \beta_0(t) + \epsilon \beta_1(t)$. Then $\beta_0$ satisfies
	\begin{equation}
		\label{eq:y_0_derivative}
		\dot{\beta_0} = A \beta_0
	\end{equation}
	with initial condition $\beta_0(0) = \id$ 
	since
	\begin{align*}
		T_t^{y_0}(x_0)
		& = \int_0^t A T_s^{y_0}(x_0)ds\\
		\Rightarrow \beta_0(t) = DT_t^{y_0}(x_0)
		& = \int_0^t D \brac{A T_s^{y_0}(x_0)} ds \\
		& = \int_0^t A DT_s^{y_0}(x_0)ds \\
		& = \int_0^t A \beta_0(s) ds
	\end{align*}
	and $\beta_1$ satisfies
	\begin{equation}
		\label{eq:y_1_derivative}
		\dot{\beta_1} = A \beta_1 + Dg(y_0(s), s)\beta_0
	\end{equation}
	with initial condition $\beta_1(0) = 0$
	since
	\begin{align*}
		T_t^{y_1}(x_0) 
		& = \int_0^t A T_s^{y_1}(x_0) + g(T_s^{y_1}(x_0),s) ds\\
		\Rightarrow \beta_1(t) = DT_t^{y_1}(x_0)
		& = \int_0^t D \brac{A T_s^{y_1}(x_0) + g(T_s^{y_0}(x_0),s)} ds \\
		& = \int_0^t A DT_s^{y_1}(x_0) + Dg(T_s^{y_0}(x_0),s) DT_s^{y_0}(x_0) ds \\
		& = \int_0^t A \beta_1(s) + Dg(y_0(s), s) \beta_0(s) ds
	\end{align*}
	
	Let $v$ be a unit vector. Define $p = (\alpha - \beta)$, let $q = \norm{pv}$. Then $q^2 = v^\intercal p^\intercal p v$. Therefore,
	\begin{align*}
		\frac{d}{dt}(q^2)
		& = 2 v^\intercal p^\intercal \dot{p} v \\
		\Rightarrow q \dot{q} 
		& = v^\intercal p^\intercal (\dot(\alpha) - \dot(\beta)) v\\
		& = v^\intercal p^\intercal \brac{A \alpha + \epsilon Dg(x,t))\alpha - A \beta_0 - \epsilon A \beta_1 - \epsilon Dg(y_0,t) \beta_0} v \\
		& = v^\intercal p^\intercal \brac{Ap + \epsilon \brac{Dg(x,t)\alpha - Dg(y_0,t)\beta_0}} v\\
		& = v^\intercal p^\intercal A pv + \epsilon v^\intercal p^\intercal \brac{Dg(x,t) \alpha - Dg(x,t)\beta_0 + Dg(x,t)\beta_0 - Dg(y_0,t) \beta_0}v\\
		& = v^\intercal p^\intercal A pv + \epsilon v^\intercal p^\intercal \brac{Dg(x,t)(\alpha - \beta) + \epsilon Dg(x,t) \beta_1 + (Dg(x,t) - Dg(y_0,t)) \beta_0} v\\
		& \le \norm{A} q^2 + \epsilon \norm{Dg(x,t)} q^2 + \epsilon^2 \norm{pv} \norm{Dg(x,t)} \norm{\beta_1v} + \epsilon \norm{pv} \norm{x - y_0} \Lip(Dg) \norm{\beta_0 v}\\
		\Rightarrow \dot{q}
		& \le \brac{\norm{A} + \epsilon \Lip(g)} q + \epsilon^2 \Lip(g) \norm{\beta_1} + \epsilon (\norm{x - y} + \norm{\epsilon y_1}) \Lip(Dg) \norm{\beta_0}\\
		& \le \brac{\norm{A} + \epsilon \Lip(g)} q + \epsilon^2 (\Lip(g) \norm{\beta_1} + (\epsilon C_0 + \norm{y_1}) \Lip(Dg) \norm{\beta_0}) \labeleqn{eq:perturbation_derivative_ode}
	\end{align*}
	where $C_0$ is the constant from \eqref{eq:pertubation_bound}. Now, it suffices to bound $\norm{\beta_0}$ and $\norm{\beta_1}$.

	For any unit vector $v$, $\norm{\beta_0 v}$ satisfies
	\begin{align*}
		\frac{d}{dt} (\norm{\beta_0 v}^2)
		& = 2 v^\intercal \beta_0^\intercal A \beta_0 v\\
		\Rightarrow \frac{d}{dt} \norm{\beta_0 v}
		& \le \norm{A} \norm{\beta_0 v}
	\end{align*}
	therefore, by \Cref{l:linear_ode_upper_bound}, $\norm{\beta_0} \le \sup_v \norm{\beta_0 v} \le \norm{\beta_0(0)} e^{2 \pi} = e^{2 \pi}$ since $\norm{A} = 1$ and $\norm{\beta_0(0)} = 1$. Similarly, for any unit vector $v$, $\norm{\beta_1v}$ satisfies
	\begin{align*}
		\frac{d}{dt} (\norm{\beta_1 v}^2)
		& = 2 v^\intercal \beta_1^\intercal (A \beta_1 + Dg(y_0,t) \beta_0) v\\
		\Rightarrow \frac{d}{dt} \norm{\beta_1 v}
		& \le \norm{A} \norm{\beta_1 v} + \norm{Dg(y_0,t)} \norm{\beta_0v}\\
		& = \norm{A} \norm{\beta_1 v} + \Lip(g) e^{2\pi} 
	\end{align*}
	therefore, by \Cref{l:linear_ode_upper_bound}, $\norm{\beta_1} \le (\Lip(g) e^{2\pi}) (e^{2\pi} - 1)$, since $\norm{A} = 1$ and $\norm{\beta_1(0)} = 0$.
	Substituting everything back into \eqref{eq:perturbation_derivative_ode}, and applying \Cref{l:linear_ode_upper_bound}, we get
	\begin{equation}
		\label{eq:perturbation_derivative_bound}
		\norm{DT_t^x(x_0) - DT_t^y(x_0)} \le \epsilon^2 \brac{\frac{\Lip(g)^2 e^{2\pi}(e^{2\pi} - 1) + (\epsilon C_0 + M_g (e^{2\pi} - 1)) \Lip(Dg) e^{2 \pi}}{1  + \epsilon \Lip(g)}} \brac{e^{2\pi (1 + \epsilon \Lip(g))}} 
	\end{equation}
	Thus, we prove the required result.
\end{proof}
\fi

\subsection{Gr\"onwall lemma}

The following lemma is very useful for bounding the growth of solutions, or errors from perturbations to ODE's.

\begin{lemma}[Gr\"onwall]\label{l:gronwall}
    If $x(t)$ is differentiable on $t\in [0,t_{\max}]$ and satisfies the differential inequality
    \begin{align*}
        \ddd t x(t) &\le ax(t) + b,
    \end{align*}
    then 
    \begin{align*}
        x(t) &\le (bt+x(0)) e^{at}
    \end{align*}
    for all $t\in [0, t_{\max}]$.
\end{lemma}

\end{document}